\documentclass[nohyperref]{article}

\usepackage{microtype}
\usepackage{graphicx}
\usepackage{subcaption}
\usepackage{booktabs} 

\usepackage{hyperref}

\usepackage[accepted]{icml2022}

\usepackage{amsmath}
\usepackage{amssymb}
\usepackage{mathtools}
\usepackage{amsthm}
\usepackage{enumitem}
\setlist[itemize]{noitemsep, topsep=0pt, leftmargin=11pt}
\usepackage{wrapfig}
\usepackage[capitalize,noabbrev]{cleveref}

\theoremstyle{plain}
\newtheorem{theorem}{Theorem}[section]
\newtheorem{proposition}[theorem]{Proposition}
\newtheorem{lemma}[theorem]{Lemma}

\newtheorem{corollary}[theorem]{Corollary}
\theoremstyle{definition}
\newtheorem{definition}[theorem]{Definition}
\newtheorem{assumption}[theorem]{Assumption}
\theoremstyle{remark}
\newtheorem{remark}[theorem]{Remark}

\def\E{\mathbb{E}}

\def\1{\mathbf{1}}
\def\P{\mathbb{P}}

\newcommand{\mc}[1]{\mathcal{#1}}

\def\imagetop#1{\begin{tabular}{l}#1\end{tabular}}

\DeclareMathOperator*{\argmin}{arg\,min}

\DeclareMathOperator*{\simp}{\mathbf{\Delta}}

\usepackage[textsize=tiny]{todonotes}

\begin{document}

\twocolumn[
\icmltitle{GALAXY: Graph-based Active Learning at the Extreme}



\icmlsetsymbol{equal}{*}

\begin{icmlauthorlist}
\icmlauthor{Jifan Zhang}{wisc}
\icmlauthor{Julian Katz-Samuels}{wisc}
\icmlauthor{Robert Nowak}{wisc}
\end{icmlauthorlist}

\icmlaffiliation{wisc}{University of Wisconsin, Madison, USA}

\icmlcorrespondingauthor{Jifan Zhang}{jifan@cs.wisc.edu}

\icmlkeywords{Active Learning, Deep Learning}

\vskip 0.3in
]



\printAffiliationsAndNotice{}  

\begin{abstract}
Active learning is a label-efficient approach to train highly effective models while interactively selecting only small subsets of unlabeled data for labelling and training. In ``open world" settings, the classes of interest can make up a small fraction of the overall dataset -- most of the data may be viewed as an out-of-distribution or irrelevant class.  This leads to extreme class-imbalance, and our theory and methods focus on this core issue.  We propose a new strategy for active learning called \texttt{\texttt{GALAXY}} (\textbf{G}raph-based \textbf{A}ctive \textbf{L}earning \textbf{A}t the e\textbf{X}tr\textbf{E}me), which blends ideas from graph-based active learning and deep learning.  \texttt{GALAXY} automatically and adaptively selects more class-balanced examples for labeling than most other methods for active learning. Our theory shows that \texttt{GALAXY} performs a refined form of uncertainty sampling that gathers a much more class-balanced dataset than vanilla uncertainty sampling.  Experimentally, we demonstrate \texttt{GALAXY}'s superiority over existing state-of-art deep active learning algorithms in unbalanced vision classification settings generated from popular datasets.

\end{abstract}

\section{Introduction}

Training deep learning systems can require enormous amounts of labeled data.  Active learning aims to reduce this burden by sequentially and adaptively selecting examples for labeling, with the goal of obtaining a relatively small dataset of especially informative examples. The most common approach to active learning is \emph{uncertainty sampling}.  The idea is to train a model based on an initial set of labeled data and then to select unlabeled examples that the model cannot classify with certainty. These examples are then labeled, the model is re-trained using them, and the process is repeated. 
Uncertainty sampling and its variants can work well when the classes are balanced.  
However, in many applications datasets may be very unbalanced, containing very rare classes or one very large majority class. As an example, suppose an insurance company would like to train an image-based machine learning system to classify various types of damage to the roofs of buildings \citep{conathan2018active}. It has a large corpus of unlabeled roof images, but the vast majority contain no damage of any sort.

Unfortunately, under extreme class imbalance, uncertainty sampling tends to select examples mostly from the dominant class(es), often leading to very slow learning. In this paper, we take a novel approach specifically targeting the class imbalance problem.  Our method is guaranteed to select examples that are both \emph{uncertain} and \emph{class-diverse}; i.e., the selected examples are relatively balanced across the classes even if the overall dataset is extremely unbalanced. In a nutshell, our algorithm sorts the examples by their softmax uncertainty scores and applies a bisection procedure to find consecutive pairs of points with differing labels. This procedure encourages finding uncertain points from a diverse set of classes. In contrast,  uncertainty sampling focuses on sampling around the model's decision boundary and therefore will collect a biased sample if this model decision boundary is strongly skewed towards one class. Figure \ref{fig:intro} displays the results of one of our experiments, showing that our proposed \texttt{GALAXY} algorithm  learns much more rapidly and collects a significantly more diverse dataset than uncertainty sampling. 

We make the following contributions in this paper:
\begin{itemize}
    \item we develop a novel, scalable algorithm \texttt{GALAXY}, tailored to the extreme class imbalance setting, which is frequently encountered in practice,
    \item \texttt{GALAXY} is easy to implement, requiring relatively simple modifications to commonplace uncertainty sampling approaches,
    \item we conduct extensive experiments showing that \texttt{GALAXY} outperforms a wide collection of deep active learning algorithms in the imbalanced settings, and 
    \item we give a theoretical analysis showing that \texttt{GALAXY} selects much more class-balanced batches of uncertain examples than traditional uncertainty sampling strategies.
\end{itemize}

\begin{figure}
    \centering
    \begin{subfigure}[t]{.49\linewidth}
        \includegraphics[width=\linewidth]{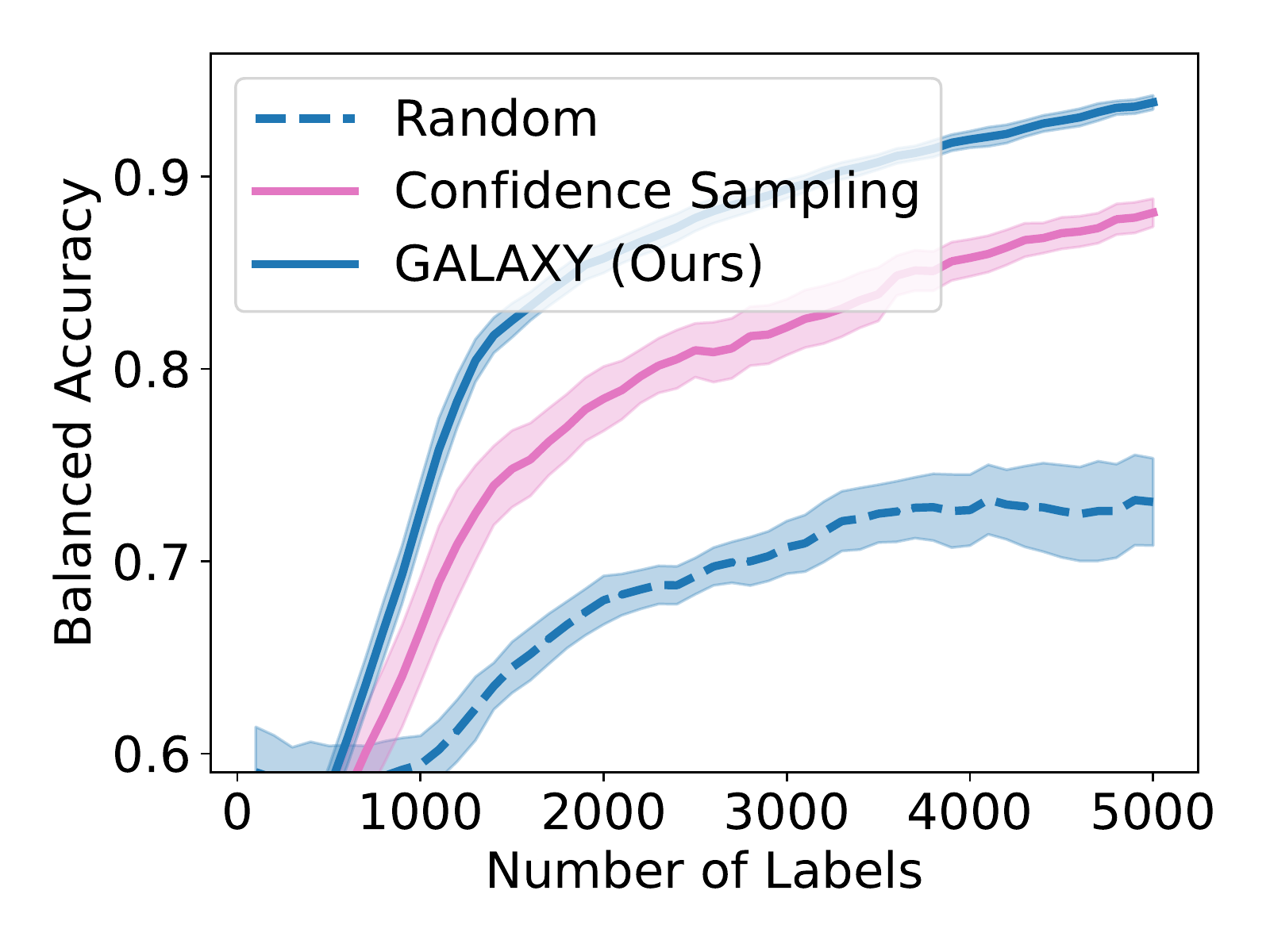}
    \end{subfigure}
    \begin{subfigure}[t]{.49\linewidth}
        \includegraphics[width=\linewidth]{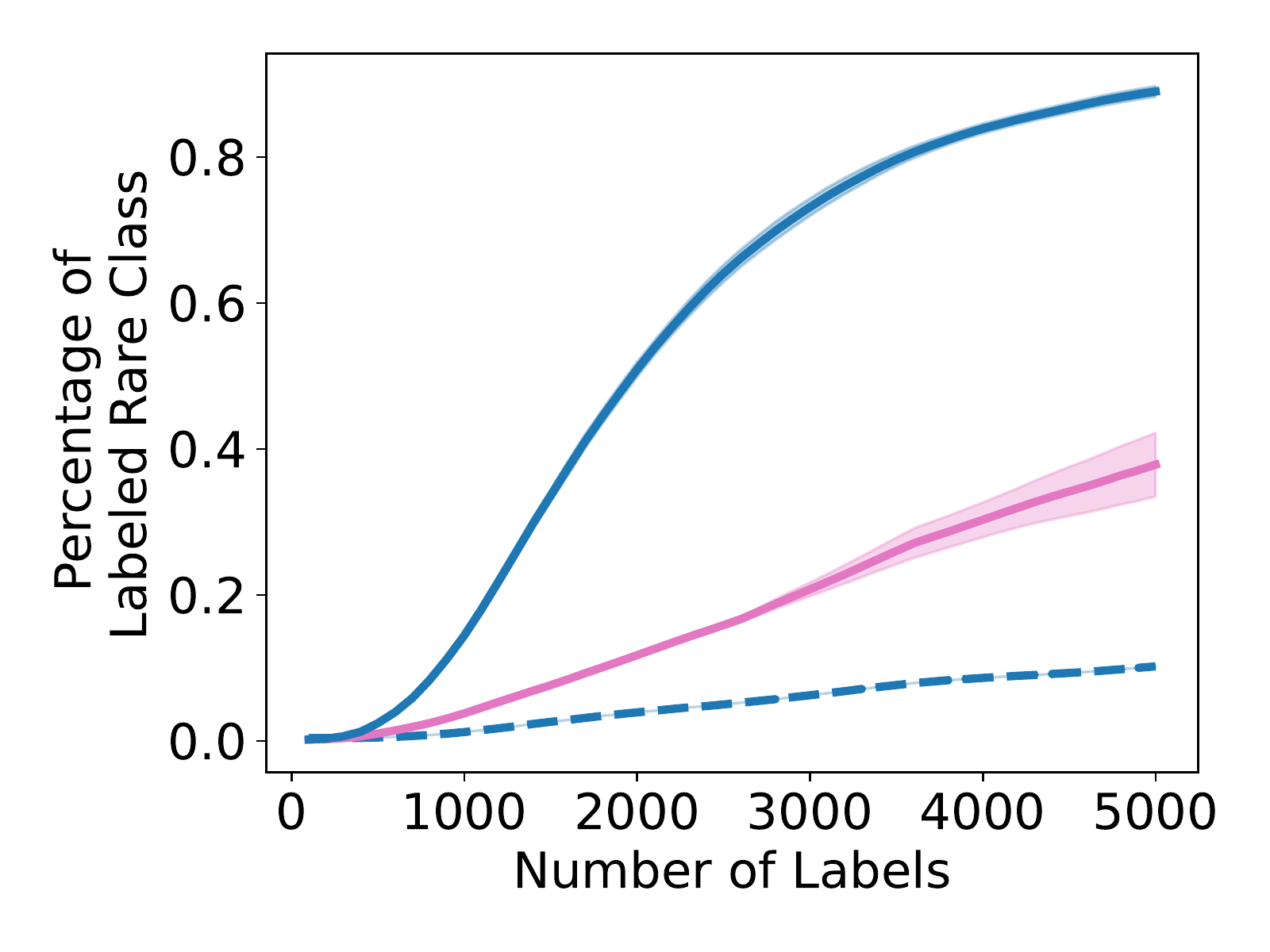}
    \end{subfigure}
    \caption{These plots depict results on a modified version of CIFAR100 with a class imbalance of $1:99$. \textbf{Left}: The plot displays the balanced accuracy of the methods where the per-class accuracy is weighted by the class size. \textbf{Right:} The plot displays the percentage of labels queried from the minority class.}
    \label{fig:intro}
\end{figure}

\section{Related Work} \label{sec:related}

\textbf{Deep Active Learning:} There are two main algorithmic approaches in deep active learning: uncertainty and diversity sampling. Uncertainty sampling queries the unlabeled examples that are most uncertain. Often, uncertainty is quantified by distance to the decision boundary of the current model (e.g., \cite{tong2001support, kremer2014active, balcan2009agnostic}). Several variants of uncertainty sampling have been proposed for deep learning (e.g., \cite{gal2017deep, ducoffe2018adversarial, beluch2018power}.

In a batch setting, uncertainty sampling often leads to querying a set of very similar examples, giving much redundant information. To deal with this issue, diversity sampling queries a batch of diverse examples that are representative of the unlabeled pool. \citet{sener2017active} propose a Coreset approach for diversity sampling for deep learning. Others include \cite{gissin2019discriminative, geifman2017deep}. However, under the class imbalance scenarios where collecting minority class examples is crucial, previous work by \citet{coleman2020similarity} has shown such methods to be less effective as they are expected to collect a subset with equal imbalance to the original dataset.

Recently, significant attention has been given to designing hybrid methods that query a batch of informative and diverse examples. \citet{ash2019deep} balances uncertainty and diversity by representing each example by its last layer gradient and aiming to select a batch of examples with large Gram determinant. \citet{citovsky2021batch} uses hierarchical agglomerative clustering to cluster the examples in the feature space and then cycles over the clusters querying the examples with smallest margin. Finally, \citet{ash2021gone} uses experimental design to find a batch of diverse and uncertain examples. 

\textbf{Class Imbalance Deep Active Learning:} A number of recent works have studied active learning in the presence of class imbalance. \citet{coleman2020similarity} proposes SEALS, a method for the setting of class imbalance and an enormous pool of unlabeled examples.
\citet{kothawade2021similar} proposes SIMILAR, which picks examples that are most similar with the collected in-distribution examples and most different from the known out-of-distribution examples. Their method achieves this by maximizing the conditional mutual information. Our setting is closest to their out-of-distribution imbalance scenario. 
Finally, \citet{emam2021active} tackles the class imablance issue by proposing BASE which queries the same number of examples per each predicted class based on margin from decision boundary, where the margin is defined by the distance to the model boundary in the feature space of the neural network. 

By contrast with the above methods, our method searches adaptively in the output space within each batch for the best threshold separating two classes and provably produces a class-balanced set of labeled examples. Adaptively searching for the best threshold is especially helpful in the extreme class imbalance setting where the decision boundary of the model is often skewed. If all labels were known, theoretically it may be possible to modify the training algorithm to obtain a model without any skew towards on class. However, this is not possible in active learning where we do not know the labels a priori. In addition, it is expensive and undesirable in practical cases to modify the training algorithm \cite{roh2020fairbatch}, making it attractive to work for any off-the-shelf training algorithm (like ours).

\textbf{Graph-based active learning:} There have been a number of proposed graph-based adaptive learning algorithms (e.g., \cite{zhu2003semi, zhu2003combining, cesa2013active, gu2012towards, dasarathy2015s2, kushnir2020diffusion}. Our work is most closely related to \cite{dasarathy2015s2}, which proposed $S^2$ a graph-based active learning with strong theoretical guarantees. While $S^2$ assumes a graph as an input and will perform badly on difficult graphs, our work builds a framework that combines the ideas of $S^2$ with deep learning to perform active learning while continually improving the graph. We review this work in more detail in Section \ref{sec:s2_review}.



\section{Problem Statement and Notation} \label{sec:notation}

We investigate the pool-based batched active learning setting, where the learner has access to a large pool of unlabeled data examples $X = \{x_1, x_2, ..., x_N\}$ and there is an unknown ground truth \emph{label} function $f^\star: X \rightarrow \{1, 2, ..., K\}$ giving the label of each example. At each iteration $t$, the learner selects a small batch of $B$ unlabeled examples $\{x_i^{(t)}\}_{i=1}^B \subseteq X $ from the pool, observes its labels $\{f^\star(x_i^{(t)})\}_{i=1}^B$, and adds the examples to $L$, the set of all the examples that have been labeled so far. After each batch of examples is queried, the learner updates the deep learning model training on all of the examples in $L$ and uses this model to inform which batch of examples are selected in the next round.


 We are particularly interested in the \emph{extreme class imbalance} problem where one class is significantly larger than the other classes. Mathematically,
\begin{align*}
    \frac{N_k}{N_K} \leq \epsilon, \ k = 1, ..., K-1.
\end{align*}
where $N_k = |\{i : f^\star(x_i) = k\}|$ denote the number of examples that belong to the $k$-th class.
Here, $\epsilon$ is some small class imbalance factor and each of the \emph{in-distribution} classes $1$, ..., $K-1$ contains far fewer examples than the \emph{out-of-distribution} $K$-th class. This models scenarios, for example in roof damage classification, self-driving and medical applications, in which a small fraction of the unlabeled examples are from classes of interest and the rest of the examples may be lumped into an ``other" or ``irrelevant" category.


Henceforth, we denote $f = \mc{A}(L)$ to be a model trained on the labeled set $L$, where $\mc{A}: \mc{L} \rightarrow \mc{F}$ is a training algorithm that trains on a labeled set and outputs a classifier.

\section{Review of $S^2$ Graph-based Active Learning}\label{sec:s2_review}
\begin{figure}
    \centering
    \begin{tabular}{l l}
        \imagetop{(a)} & \imagetop{\includegraphics[width=.8\linewidth]{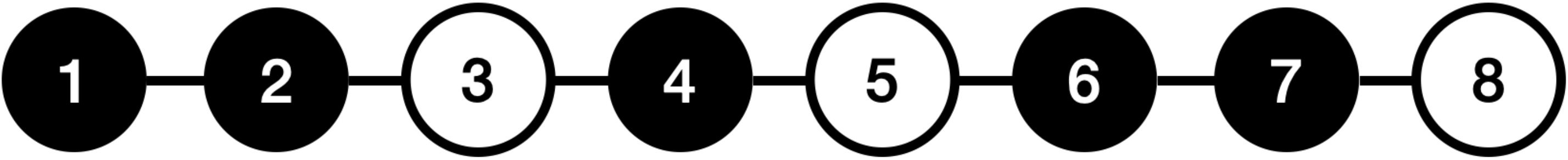}} \\
        \imagetop{(b)} & \imagetop{\includegraphics[width=.8\linewidth]{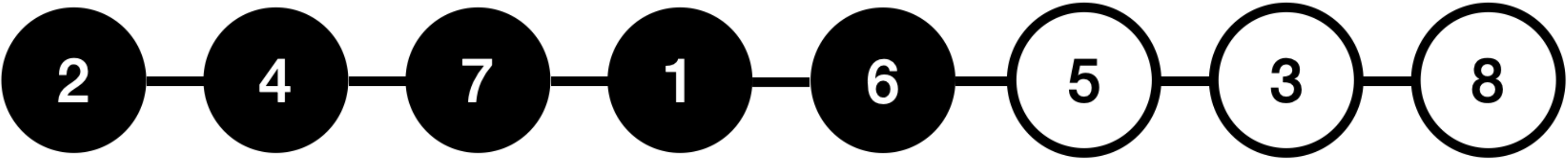}} \\
        \imagetop{(c)} & \imagetop{\includegraphics[width=.8\linewidth]{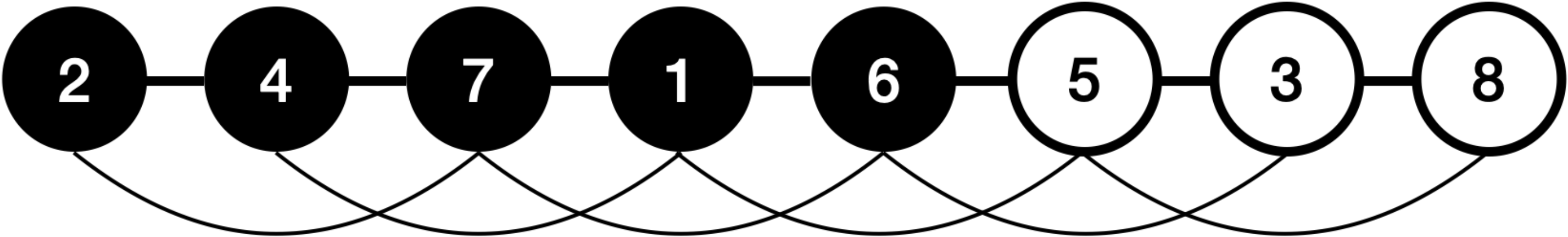}}
    \end{tabular}
    \caption{All three graphs contain the same eight numbered examples but connected in different ways. The ground truth binary labels are represented by the black and white coloring of the examples. As a result, each of their cut boundaries are: (a) $\partial C = \{2, 3, 4, 5, 6, 7, 8\}$, (b) $\partial C = \{6, 5\}$ and (c) $\partial C = \{1, 6, 5, 3\}$.}
    \label{fig:s2_demo}
\end{figure}
To begin, we introduce some notation. With slight abuse of notation, we define a undirected graph over the pool as $G = (X, E)$ with vertex set $X$ and edge set $E$, where each node $x$ in the graph is also an example in the pool $X$. Let $P_{ij}(X, E) \subset E$ denote the shortest path connecting $x_i$ and $x_j$ in the graph $G = (X, E)$, and let $|P_{ij}(X, E)|$ denote its length.\footnote{In the special case when $x_i$ and $x_j$ are not connected, we define $|P_{ij}(X, E)| = \infty$.}

\citet{dasarathy2015s2} proposed a graph-based active learning algorithm $S^2$ (see Algorithm~\ref{alg:s2}) that aims to identify all the cuts $C = \{(x, y) \in E : f^\star(x) \neq f^\star(y)\}$, namely every edge that connects an oppositely labeled pair of examples. In particular, if one labeled all of the examples in the cut boundary $\partial C = \{x \in X : \exists e \in C, x \in e\}$, one would be able to classify every example in the pool correctly. As an example, take the linear graph in Figure~\ref{fig:s2_demo}(a) where each node represents a numbered example and is associated with a binary label (black/white). It is thus necessary to query at least seven examples to identify the cut boundary and therefore the labeling.

 $S^2$ performs an alternating two phased procedure. First, if every pair of connected examples have the same label, $S^2$ queries an unlabeled example uniform at random. Second, whenever there exist paths connecting examples with different labels, the algorithm bisects along the shortest among these paths until it identifies a cut. The algorithm then removes the identified edge from the graph.
 
 \citet{dasarathy2015s2} has shown that the PAC sample complexity to identify all cuts highly depends on the input graph's structural properties. As an example consider Figure~\ref{fig:s2_demo}, which depicts several graphs on the same set of examples. In graph~(a), one needs to query at least seven examples, while in graph~(b) one need only query two examples. Indeed, such a difference can be made arbitrarily large. The work of \citet{dasarathy2015s2}, however, did not address the major problem of how to obtain an ``easier" graph that requires fewer examples queries for active learning.


\begin{algorithm}
\begin{algorithmic}
\STATE \textbf{Input: } Graph $G = (X, E)$, total budget $2 \leq M \leq N$

\STATE \textbf{Initialize: } Labeled set $L = \{x, y\}$ where $x \neq y$ are uniform random samples from $X$

\FOR{$t = 1, 2, ..., M$}

    \STATE $i^\star, j^\star \leftarrow \argmin_{i, j: (x_i, x_j \in L) \land (f^\star(x_i)\neq f^\star(x_j))} P_{ij}(X, E)$

    \IF{$|P_{i^\star j^\star}(X, E)| = \infty$}
        \STATE Query $x \sim \text{Unif}(X \backslash L)$
    \ELSE
        \STATE Query the mid point $x$ of $P_{i^\star j^\star}(X, E)$
    \ENDIF

    \STATE Update labeled set: $L \leftarrow L \cup \{x\}$
    \STATE Remove cuts from current graph:\\ $E \leftarrow E \backslash \{(x, y)\in E: (y\in L) \land (f^\star(x) \neq f^\star(y))\}$ 
\ENDFOR

\STATE \textbf{Return: } Labeled set $L$

\end{algorithmic}
\caption{$S^2$: Shortest Shortest Path}
\label{alg:s2}
\end{algorithm}

\section{\texttt{GALAXY}}
Our algorithm \texttt{GALAXY} shown in Algorithm~\ref{alg:galaxy} blends graph-based active learning and deep active learning through the following two alternating steps
\begin{itemize}
    \item Given a trained neural network, we construct a graph based on the neural network's predictions and apply a modified version of $S^2$ to it, efficiently collecting an informative batch of labels (Algorithm~\ref{alg:galaxy}).

    \item Given a new batch of labeled examples, we train a better neural network model that will be used to construct a better graph for active learning (Algorithm~\ref{alg:build_graph}).

\end{itemize}
    

To construct graphs from a learned neural network in multi-class settings, we take a one-vs-all approach on the output (softmax) space as shown in Algorithm~\ref{alg:build_graph}. For each class $k$, we build a linear graph $G^{(k)}$ by ranking the model's confidence margin $\delta^{(k)}_i$ on each example $x_i \in X$. For a neural network $f_\theta$, the confidence margin is simply defined as $\delta^{(k)}_i = [f_\theta(x_i)]_k - \max_{k'}[f_\theta(x_i)]_{k'}$, where $[\cdot]_k$ denotes the $k$-th element of the softmax vector. We break ties by the confidence scores $[f_\theta(x_i)]_k$ themselves (equivalently by $\max_{k'}[f_\theta(x_i)]_{k'}$). 
Intuitively, for each graph $G^{(k)}$, we sort examples according to their likelihood to belong to class $k$.
Indeed, when $f_\theta$ is a perfect classifier on the pool, each linear graph constructed behaves like Figure~\ref{fig:s2_demo}(b), i.e., every example in class $k$ is perfectly separated from all other classes with only one cut in between.
\begin{algorithm}
\begin{algorithmic}
\STATE \textbf{Input: } Pool $X$, neural network $f_\theta: \mc{X} \to \simp^{(K - 1)}$
\STATE Confidence for each $i \in [N]$: $q_i \leftarrow \max_{k \in [K]} [f_\theta(x_i)]_k$
\FOR{$k = 1, ..., K$}
    \STATE Compute margins $\delta_i^{(k)} \leftarrow [f_\theta(x_i)]_k - q_i$
    
    \STATE Sort by margin and break ties by confidence: $A^{(k)} = \{\alpha_i^{(k)} \in [N]\}_{i=1}^N$ is a permutation of $[N]$ and denotes an ordering index set such that $\forall i < N$
    \begin{align*}
        \left(\delta_{\alpha_i^{(k)}} \leq \delta_{\alpha_{i+1}^{(k)}}\right) \land \left(\delta_{\alpha_i^{(k)}} = \delta_{\alpha_{i+1}^{(k)}} \Rightarrow q_{\alpha_i^{(k)}} \leq q_{\alpha_{i+1}^{(k)}}\right)
    \end{align*}
    
    \STATE Connect edges $E^{(k)} \leftarrow \{(x_{\alpha_i^{(k)}}, x_{\alpha_{i+1}^{(k)}}) : i \in [N - 1]\}$
    
\ENDFOR

\STATE \textbf{Return: } Graphs $\{G^{(k)} = (X, E^{(k)})\}_{k=1}^K$, rankings $\{A^{(k)}\}_{k=1}^K$

\end{algorithmic}
\caption{Build Graph}
\label{alg:build_graph}
\end{algorithm}

Our algorithm \texttt{GALAXY} shown in Algorithm~\ref{alg:galaxy} proceeds in a batched style. For each batch, \texttt{GALAXY} first trains a neural network to obtain graphs constructed by the procedure described above. It then performs $S^2$ style bisection-like queries on all of the graphs but with two major differences.
\begin{itemize}
    \item To accommodate multiple graphs, we treat each linear graph $G^{(k)}$ as a binary one-vs-all graph, where we gather all shortest paths $P_{ij}(X, E^{(k)})$ that connects a queried example in class $k$ and a queried example in any other classes. If such shortest paths exist, we then find the shortest of these shortest paths across \emph{all} $k \in [K]$ and the bisect the resulting shortest shortest path like in $S^2$.
    
    \item When no such shortest path exists, instead of querying an example uniform at random as in $S^2$, we increase the order of the graphs by Algorithm~\ref{alg:connect} and perform bisection procedures on the updated graphs. Here, we refer to an $m$-th order linear graph where each example is connected to all of its neighboring $m$ examples from each side. For example, Figure~\ref{fig:s2_demo}(c) shows a graph of order $2$ as opposed to an order $1$ graph shown in Figure~\ref{fig:s2_demo}(b). Intuitively, bisecting after the \texttt{Connect} procedure is equivalent with querying around the discovered cuts. For example in the case of Figure~\ref{fig:s2_demo}(b), after querying examples $5$ and $6$, our algorithm will connect second order edges and query exactly examples $1$ and $3$ as the next two queries.
\end{itemize}

\begin{algorithm}
\begin{algorithmic}
\STATE \textbf{Input: } Graphs $\{G^{(k)} = (X, E^{(k)})\}_{k=1}^K$, rankings $\{A^{(k)}\}_{k=1}^K$, edge order $ord$

\FOR{$k = 1, ..., K$}
    \STATE $E^{(k)} \leftarrow E^{(k)} \cup \{(x_{\alpha_i^{(k)}}, x_{\alpha_{i+ord}^{(k)}})\}_{i=1}^{N - ord}$
\ENDFOR

\STATE \textbf{Return: } Graphs $\{G^{(k)} = (X, E^{(k)})\}_{k=1}^K$
\caption{Connect: build higher order edges}
\label{alg:connect}

\end{algorithmic}
\end{algorithm}

\begin{algorithm}
\begin{algorithmic}
\STATE\textbf{Input: } Pool $X$, neural network training algorithm $\mc{A}: \mc{L} \rightarrow \mc{F}$, number of rounds $T$, batch size $B$ ($TB \leq |X|$)
\STATE\textbf{Initialize: } Uniformly sample $B$ elements without replacement from $X$ to form $L$

\FOR{$t = 1, ..., T - 1$}
    \STATE Train neural network: $f_\theta \leftarrow \mc{A}(L)$

    \STATE $\{G^{(k)}\}_{k=1}^K, \{A^{(k)}\}_{k=1}^K \leftarrow \texttt{Build\_Graph}(X, f_\theta)$

    \STATE Graph order: $ord \leftarrow 1$
    
    \FOR{$s = 1, ..., B$}
        \STATE Find shortest shortest path among all graphs:
        \vspace{-.5\intextsep}
        \begin{align} \label{eqn:s3_path}
            i^\star,& j^\star, k^\star \leftarrow \argmin_{\substack{i, j, k: (x_i, x_j \in L) \land\\ (f^\star(x_i) = k, f^\star(x_j)\neq k)}} P_{ij}(X, E^{(k)})
        \end{align} 
        \vspace{-.5\intextsep}
        \IF{$|P_{i^\star j^\star}(X, E^{(k^\star)})| = \infty$}
            \STATE $\{G^{(k)}\} \leftarrow \texttt{Connect}(\{G^{(k)}\}, \{A^{(k)}\}, ord + 1)$
            \STATE Recompute $i^\star, j^\star, k^\star$ by \eqref{eqn:s3_path}
            \STATE $ord \leftarrow ord + 1$
        \ENDIF
        \STATE Query the mid point $x$ of $P_{i^\star j^\star}(X, E^{(k^\star)})$
    
        \STATE Update labeled set: $L \leftarrow L \cup \{x\}$
        \STATE Remove cuts for each $G^{(k)}, k\in[K]$: $E^{(k)} \leftarrow E^{(k)} \backslash \{(x, y)\in E^{(k)}: (y\in L) \land (f^\star(x) \neq f^\star(y))\}$ 
    \ENDFOR

\ENDFOR
\STATE \textbf{Return: } Final classifier $f_\theta \leftarrow \mc{A}(L)$
\end{algorithmic}
\caption{\texttt{GALAXY}}
\label{alg:galaxy}
\end{algorithm}

\section{Analysis}
\begin{figure*}
    \centering
    \begin{tabular}{l l}
        \imagetop{(a)} & \imagetop{\includegraphics[width=.7\linewidth]{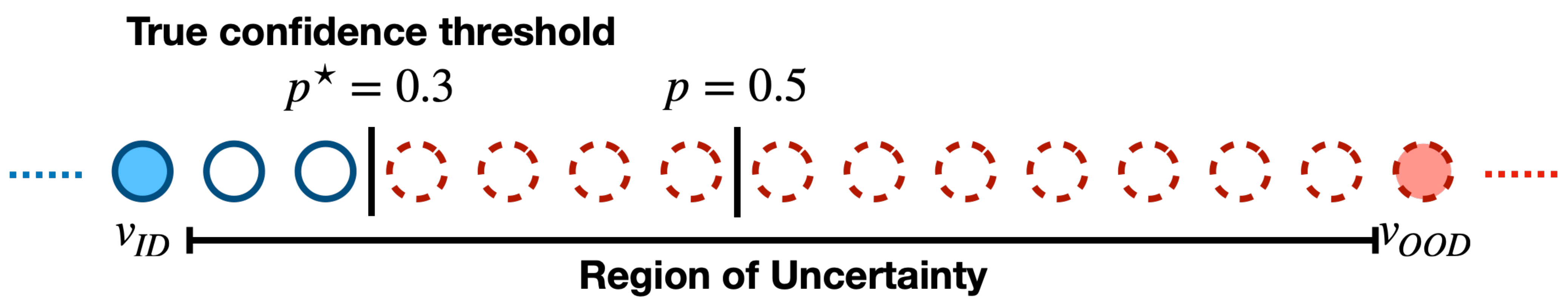}} \\
        & \\
        \imagetop{(b)} & \imagetop{\includegraphics[width=.7\linewidth]{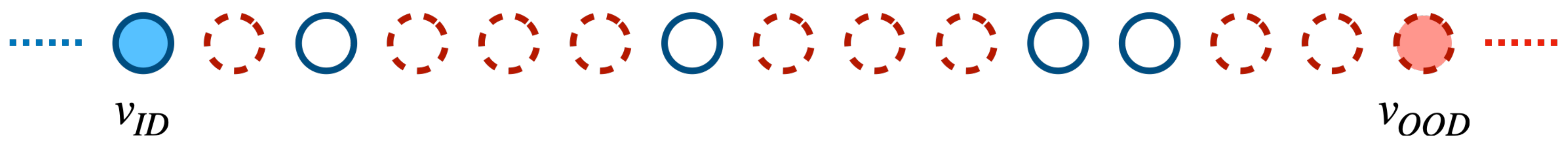}}
    \end{tabular}
    \caption{(a) and (b) denotes two different linear graphs generated from two different classifiers by ranking their corresponding confidence scores. The ground truth label of each example is represented by its border -- solid blue lines for class \texttt{ID} and dotted red lines for class \texttt{OOD}. The linear graph in (a) is a \emph{separable} graph where all examples in class \texttt{ID} are of low confidence scores while class \texttt{OOD} examples have higher confidence scores. By contrast, the linear graph in (b) is \emph{non-separable}.}
    \label{fig:analysis}
\end{figure*}

\subsection{\texttt{GALAXY} at the Extreme}
In this section, we analyze the behavior of \texttt{GALAXY} in the two-class setting and specifically when class $\texttt{OOD}$ (out-of-distribution) has much more examples than class $\texttt{ID}$ (in-distribution). In a binary separable case, we bound expected batch balancedness of both the bisection procedure and \texttt{GALAXY}, whereas uncertainty sampling could fail to sample any \texttt{ID} examples at all. At the end we also show a noise tolerance guarantee that \texttt{GALAXY} will find the optimal uncertainty threshold with high probability. For proper indexing below, we let $\texttt{OOD} = 1$ and $\texttt{ID} = 2$. 

\textbf{Reduction to single linear graph.} 
Recall in Algorithm~\ref{alg:build_graph}, we build a graph for each class by sorting the margin scores of that class on the pool. In the binary classification case, it is sufficient to consider one graph generated from sorting confidence scores. This follows due to the symmetry of the two graphs in the binary case.

\textbf{Universal approximator and region of uncertainty.}
Since neural networks are universal approximators, we make the following assumption.
\begin{assumption}
Given a labeled subset $L$ of $X$, let $f_\theta = \mc{A}(L)$ be the neural network classifier trained on $L$. We assume $f_\theta$ classifies every example in $L$ perfectly. Namely, $\forall x \in L, f^\star(x) = \texttt{OOD} \iff [f_\theta(x)]_{\texttt{OOD}} > 0.5$.
\end{assumption}
\begin{definition}
Let $v_{\texttt{ID}}$ denote the labeled example in class \texttt{ID} with the highest confidence and $v_{\texttt{OOD}}$ denote the labeled example in class \texttt{OOD} with the lowest confidence. We then define all examples in between, i.e. $\{x \in X: [f_\theta(v_{\texttt{ID}})]_{\texttt{OOD}} < [f_\theta(x)]_{\texttt{OOD}} < [f_\theta(v_{\texttt{OOD}})]_{\texttt{OOD}}\}$, to be the \emph{region of uncertainty}.
\end{definition}
In practice, since the neural network model should be rather certain in its predictions on the labeled set, we expect the region of uncertainty to be relatively large. We show an example in Figure~\ref{fig:analysis} where filled circles represent the labeled examples. The filled blue example on the left is $v_{\texttt{ID}}$ and the filled red example on the right is $v_{\texttt{OOD}}$. The region of uncertainty are then all of the examples in between.

In the following, we first derive our balancedness results in the separable case such as in Figure~\ref{fig:analysis}(a) and turn to noise tolerance analysis in the end. First in the separable case, we let $n_{\texttt{ID}}$ denote the number of in-distribution examples and $n_{\texttt{OOD}}$ denote the number of out-of-distribution examples both in the region of uncertainty. First we analyze the bisection following procedure that adaptively finds the true uncertainty threshold (cut in the separable linear graph).
\begin{definition} \label{def:bisection}
Our bisection procedure works as follows when given region of uncertainty with $n_{\texttt{ID}} + n_{\texttt{OOD}}$ examples.
\begin{itemize}
    \item Let $m$ represent the number of examples in the latest region of uncertainty, query the $i-th$ example based on the sorted uncertainty scores. Here, $i = \lfloor\frac{m}{2}\rfloor + 1$ or $i = \lceil\frac{m}{2}\rceil$ with equal probability.
    
    
    
    \item If observe \texttt{ID}, update the region of uncertainty to be examples ranked $\{i + 1, ..., n_{\texttt{ID}} + n_{\texttt{OOD}}\}$ based on uncertainty scores. Recurse on the new region of uncertainty. 
    
    Similarly, if observe \texttt{OOD}, update the region of uncertainty to be examples ranked $\{1, ..., i - 1\}$ based on uncertainty scores. Recurse on the new region of uncertainty. 
    
    \item Terminate once the region of uncertainty is empty ($m=0$).
\end{itemize}
\end{definition}
The exact number of labels collected from the \texttt{ID} and \texttt{OOD} classes depends on the specific numbers of examples in the region of uncertainty.  We characterize the generic behavior of the biscection process with a simple probabilistic model showing the following theorem that the method tends to find balanced examples among both classes. Proofs of the following results appear in the Appendices.
\begin{theorem}[Sample Balancedness of Bisection] \label{thm:bisection}
Assume $n_{\texttt{ID}} + n_{\texttt{OOD}} \geq 2^{z+2} - 1$ for some $z\geq 1$ and that the examples labeled in the first $z$ bisection steps are all from class \texttt{OOD}.  At least $n' \geq 3$ examples remain in the region of uncertainty and suppose that $n_{\texttt{ID}} \sim \text{Unif}(\{1, ..., n' - 1\})$. If we let $m_{\texttt{ID}}$ and $m_{\texttt{OOD}}$ be the number of queries in each of the \texttt{ID} and \texttt{OOD} classes made by the bisection procedure described in Definition~\ref{def:bisection}, we must have
\begin{align*}
    \frac{\E[m_{\texttt{ID}}]}{\E[m_{\texttt{OOD}}]} \geq \frac{\frac{1}{2}\log_2 (n')}{z + \frac{1}{2} \log_2 (n')}
\end{align*}
where the expectations are with respect to the uniform distribution above.
\end{theorem}
The unbalancedness factor of the region of uncertainty is at most $\frac{n_{\texttt{ID}}}{n_{\texttt{OOD}}} \leq \frac{1}{2^z}$. When $z$ is large, we must have $\frac{\E[m_{\texttt{ID}}]}{\E[m_{\texttt{OOD}}]} \geq \frac{1}{z} \gg \frac{n_{\texttt{ID}}}{n_{\texttt{OOD}}}$. Thus, the bisection procedure collects a batch that \emph{improves on the unbalanced factor exponentially}.

Next, we characterize the balancedness of the full \texttt{GALAXY} algorithm. When running \texttt{GALAXY} on a separable linear graph, it is equivalent with first running bisection procedure to find the optimal uncertainty threshold, followed by querying around the two sides of the threshold equally. We therefore incorporate our previous analysis on the bisection procedure and especially focus on the second part where one queries around the optimal uncertainty threshold.

\begin{corollary}[Sample Balancedness of Batched \texttt{GALAXY}, Proof in Appendix~\ref{sec:galaxy_proof}] \label{thm:galaxy_balance}
Assume $n_{\texttt{ID}}$ and $n_{\texttt{OOD}}$ are under same noiseless setting as in Theorem~\ref{thm:bisection}. If \texttt{GALAXY} takes an additional $B' < n'$ queries after the bisection procedure terminates, so that $B = B' + \lceil\log_2(n_{\texttt{ID}}  + n_{\texttt{OOD}})\rceil$ examples are labeled in total and if we let $m_{\texttt{ID}}$ and $m_{\texttt{OOD}}$ be the number of queries in each class made by \texttt{GALAXY}, we must have
\begin{align}
    \frac{\E[m_{\texttt{ID}}]}{\E[m_{\texttt{OOD}}]} &\geq \frac{y}{B - y} \nonumber \geq \frac{y}{z + 5y + 3}
\end{align}
where $y = \max\{\lfloor \frac{B'}{4}\rfloor, \frac{1}{2}\log_2 (n')\}$ and the expectations are with respect to the uniform distribution in $n_{\texttt{ID}}$.
\end{corollary}
In the above theorem, since $z < \log_2(n_{\texttt{ID}}  + n_{\texttt{OOD}})$ when $B'$ is large, we can then recover a constant factor of balancedness.
On the other hand, uncertainty sampling does not enjoy the same balancedness guarantees when the model decision boundary is biased towards the \texttt{OOD} class.

\begin{proposition}[Sample Balancedness of Uncertainty Sampling] \label{thm:uncertainty}
Assume $n_{\texttt{ID}}$ and $n_{\texttt{OOD}}$ are under same noiseless setting as in Theorem~\ref{thm:bisection}. If we let $m_{\texttt{ID}}$ and $m_{\texttt{OOD}}$ be the number of queries in each of the \texttt{ID} and \texttt{OOD} classes made by an uncertainty sampling procedure with batch size $B < n'$ steps, we have
\begin{align*}
    \min_{p_\star}\frac{\E[m_{\texttt{ID}}]}{\E[m_{\texttt{OOD}}]} = 0
\end{align*}
where the expectations are with respect to $n_{\texttt{ID}}$. The minimization is taken over the true confidence threshold $p_\star$ where the classification accuracy is maximized.
\end{proposition}
Note that the number of queries collected by uncertainty sampling, $m_{\texttt{ID}}$ and $m_{\texttt{OOD}}$, inherently depends on $p_\star$. The above proposition can been seen as demonstrated by Figure~\ref{fig:analysis}, where when training a model under extreme imbalance, the model could be biased towards \texttt{OOD} and thus the \emph{true confidence threshold} $p^\star \neq 0.5$. Since $B < n' \leq n_{\texttt{OOD}}$, in the worst case, uncertainty sampling could have selected a batch all in \texttt{OOD} regardless of the value $n_{\texttt{ID}}$ takes. Therefore, in such cases, we have $\frac{\E[m_{\texttt{OOD}}]}{\E[m_{\texttt{ID}}]} = 0$.


We will now show \texttt{GALAXY}'s robustness in non-separable graphs. We model the noises by randomly flipping the true labels of a separable graph.

\begin{theorem}[Noise Tolerance of \texttt{GALAXY}, Proof in Appendix~\ref{sec:noise_proof}] \label{thm:galaxy_noise}
Let $n = n_{\texttt{ID}} + n_{\texttt{OOD}}$. Suppose the true label of each example in the region of uncertainty is corrupted independently with probability $\frac{\delta}{\lceil \log_2 n \rceil}$. Let $B$ denote the batch size of \texttt{GALAXY}, $m_{\texttt{ID}}$ and $m_{\texttt{OOD}}$ be the number of queries in each class made by \texttt{GALAXY}, with probability at least $1 - \delta$ we have
\begin{align*}
    \frac{\E[m_{\texttt{ID}}]}{\E[m_{\texttt{OOD}}]} \geq \frac{\frac{1}{2}\log_2 (n')}{B - \frac{1}{2} \log_2 (n')}
\end{align*}
where the expectations are with respect to $n_{\texttt{ID}}$.
\end{theorem}
Note in practice batch size in active learning is usually small. When $B \approx 2\log_2 n'$, the above result also implies that with about $n \cdot \frac{\delta}{\log_2 n}$ labels corrupted at random, \texttt{GALAXY} collects a balanced batch with probability at least $1-\delta$.

\subsection{Time Complexity}
We compare the per-batch running time of \texttt{GALAXY} with confidence sampling, showing that they are comparable in practice.
Recall that $B $ is the batch size, $N$ is the pool size and $K$ is the number of classes. Let $Q$ denote the forward inference time of a neural network on a single example.

Confidence sampling has running time $O(QN + KN + B\log N)$, where $O(QN)$ comes from forward passes on the entire pool, $O(KN)$ comes from computing the maximum confidence of each example and $O(B\log N)$ is the time complexity of choosing the top-B examples according to uncertainty scores. On the other hand, our algorithm \texttt{GALAXY} has time $O(QN + KN\log N + BKN)$. Here, $O(KN\log N)$ is the complexity of constructing $K$ linear graphs (Algorithm~\ref{alg:build_graph}) by sorting through margin scores and $O(BKN)$ comes from finding the shortest shortest path, for $B$ elements among $K$ graphs.

In practice $O(QN)$ dominates all of the other terms, so making these running times comparable. Indeed, in all of our experiments conducted in Section~\ref{sec:exp_results}, \texttt{GALAXY} is less than 5\% slower when compared to confidence sampling.

\section{Experiments}
We conduct experiments under $8$ different class imbalance settings. These settings are generated from three image classification datasets with various class imbalance factors. If the classes are balanced in the dataset, then most active learning strategies (including \texttt{GALAXY}) perform similarly, with relative small differences in performance, so we focus our presentation on unbalanced situations. We will first describe the setups (Section~\ref{sec:exp_setup}) before turning to the results in Section~\ref{sec:exp_results}. Finally, we present a comparison with vanilla $S^2$ algorithm and demonstrate the importance of reconstructing the graphs in Section~\ref{sec:exp_compare}.\footnote{Code can be found in \url{https://github.com/jifanz/GALAXY}.}

\subsection{Setup} \label{sec:exp_setup}
We use the following metric and training algorithm to reweight each class by its number of examples. By doing this, we downweight significantly the large ``other'' class while not ignoring it completely. More formally, we state our metric and training objective below.

\textbf{Metric: } Given a fixed batch size $B$ and after $T$ iterations, let $L \subset X$ denote the labeled set after the final iteration. Let $f = \mc{A}(L)$ be a model trained on the labeled set. We wish to maximize the \emph{balanced accuracy} over the pool
\begin{align*}
    ACC_{bal} &= \frac{1}{K}\sum_{k=1}^K \P(f(x) = f^\star(x) | f^\star(x) = k) \nonumber\\
    &= \frac{1}{K}\sum_{k=1}^K\left[\frac{1}{N_k}\sum_{i: f^\star(x_i)=k} \1\{f(x_i) = k\} \right]
\end{align*}
Recall that $N_k = |\{i : f^\star(x_i) = k\}|$ is the number of examples in class $k$. In all of our experiments, we set $B = 100$ and $T = 50$.
\begin{remark}
Finding good active classifiers on the pool is closely related to finding good classifiers that generalizes. See \citet{boucheron2005theory} for standard generalization bounds or \citet{katz2021improved} for a detailed discussion.
\end{remark}

\textbf{Training Algorithm $\mc{A}$:} Our training algorithm takes a labeled set $L$ as input. Let $N_k(L) = |\{x \in L : f^\star(x) = k\}|$ denote the number of labeled examples in class $k$, we use a cross entropy loss weighted by $\frac{1}{N_k(L)}$ for each class $k$. Note unlike the evaluation metric, we do not directly reweight the classes by $\frac{1}{N_k}$, as the active learning algorithms only have knowledge of labels of $L$ in practice.
Furthermore for all experiments, we use the ResNet-18 model in PyTorch pretrained on ImageNet for initialization and cold-start the training for every labeled set $L$. We use the Adam optimization algorithm with learning rate of $10^{-2}$ and a fixed $500$ epochs for each $L$.

\begin{figure*}[t!]
    \centering
    \begin{subfigure}[t]{.33\textwidth}
        \centering
        \includegraphics[width=\linewidth]{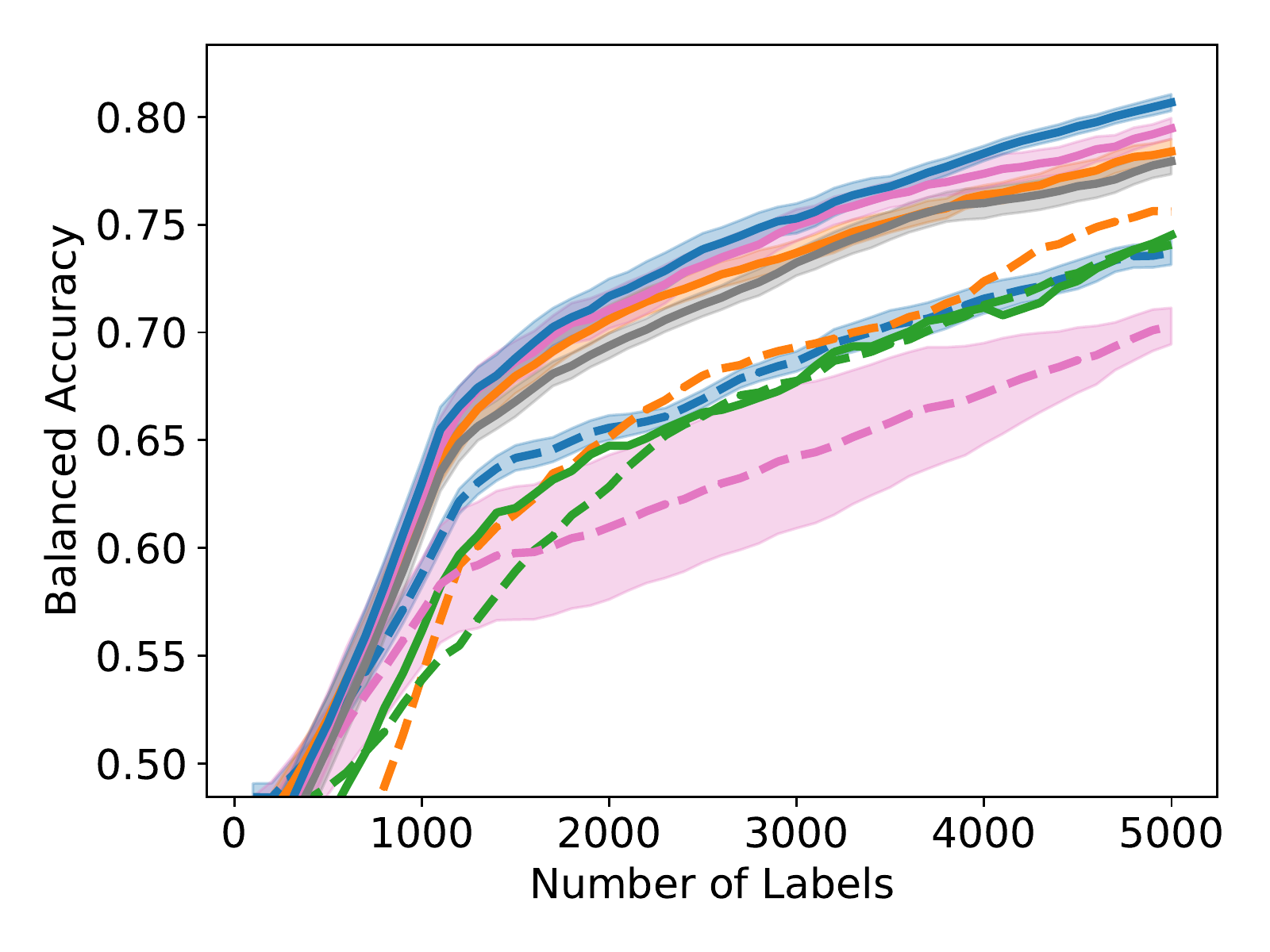}
        \caption{$ACC_{bal}$, CIFAR-10, 3 classes}
    \end{subfigure}
    \begin{subfigure}[t]{.33\textwidth}
        \centering
        \includegraphics[width=\linewidth]{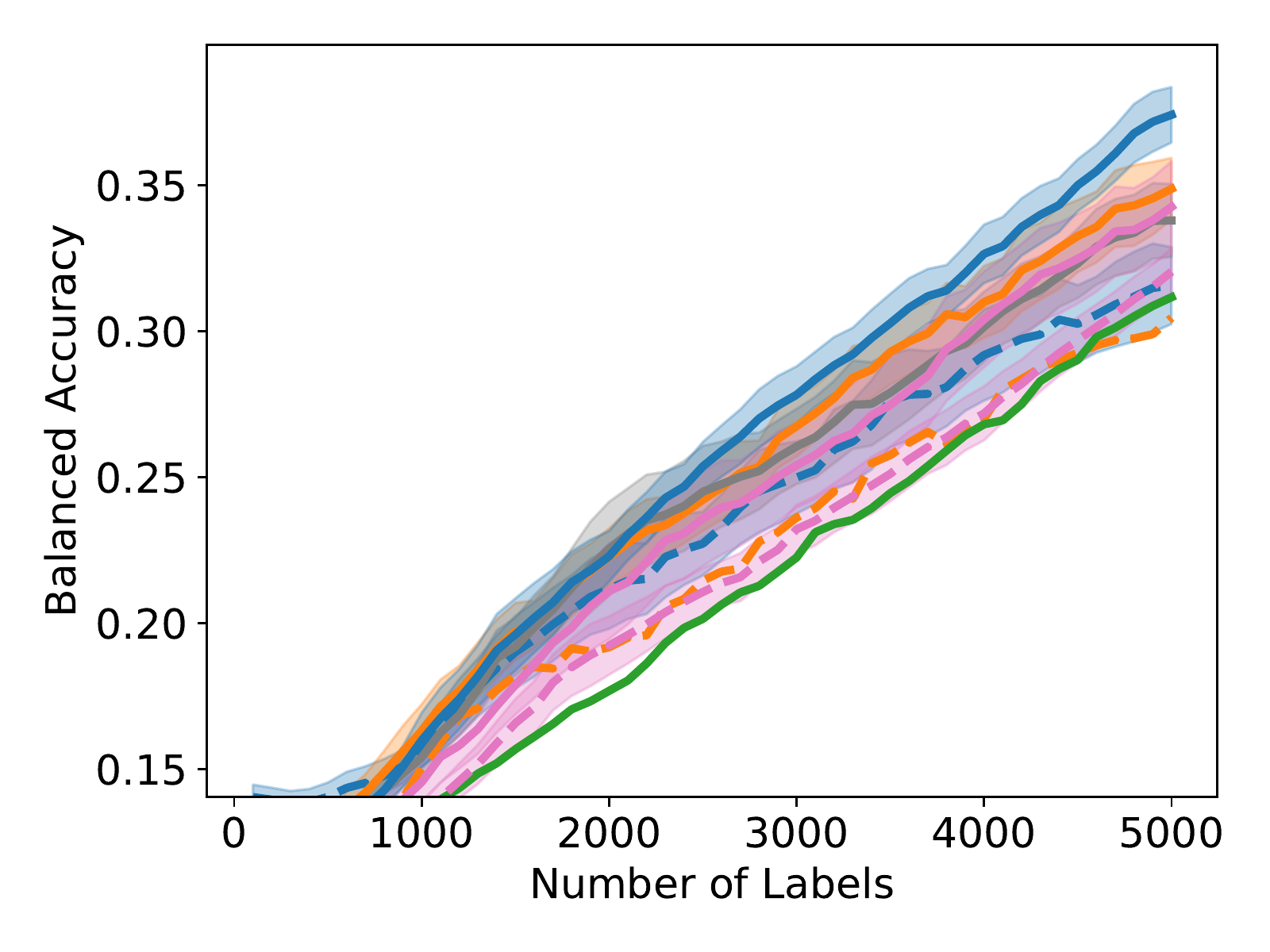}
        \caption{$ACC_{bal}$, CIFAR-100, 10 classes}
    \end{subfigure}
    \begin{subfigure}[t]{.33\textwidth}
        \centering
        \includegraphics[width=\linewidth]{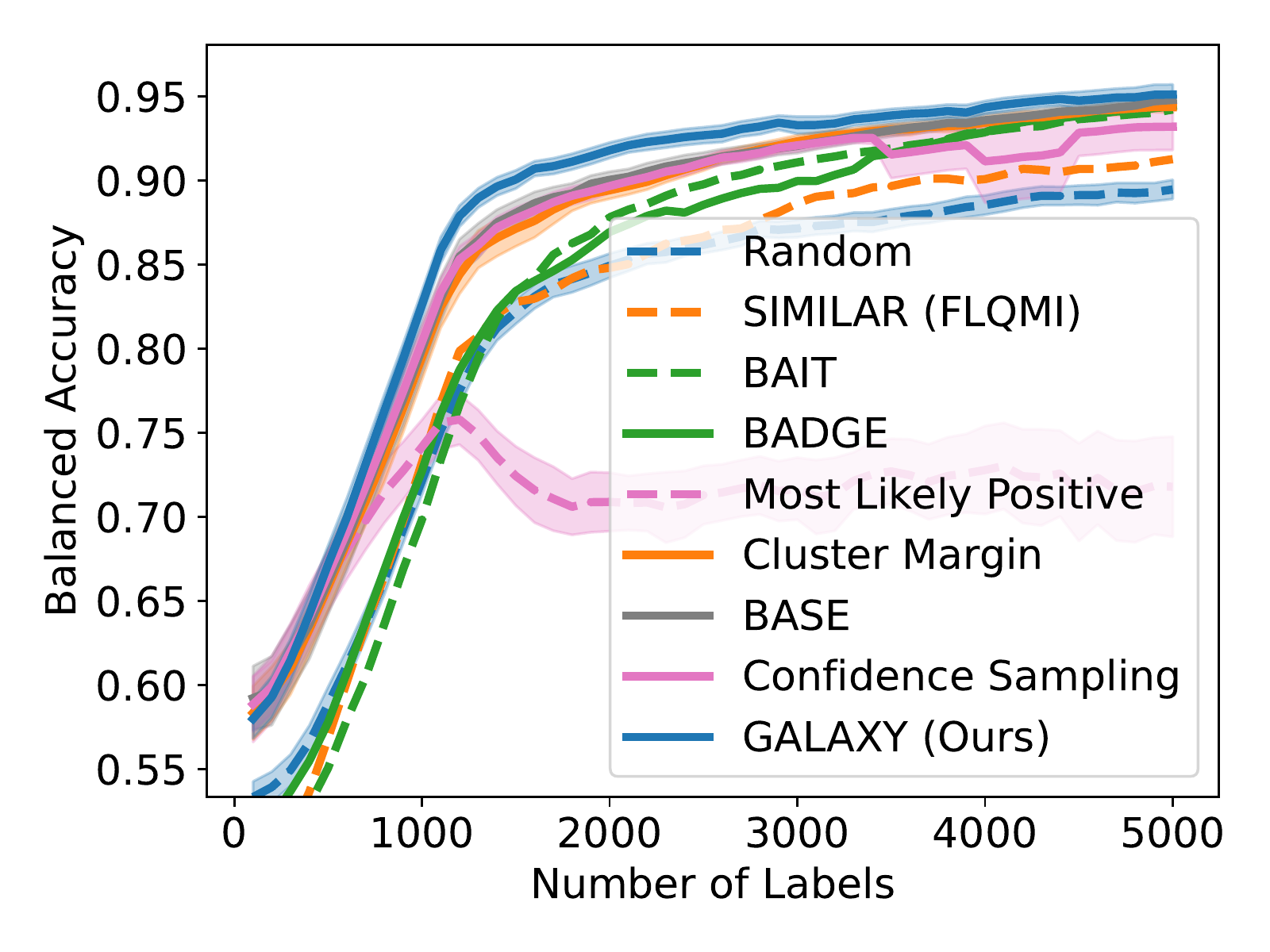}
        \caption{$ACC_{bal}$, SVHN, 2 classes}
    \end{subfigure}
    \caption{Performance of \texttt{GALAXY} against baselines on selected settings. Legend shown in (c) is shared across all three plots.}
    \label{fig:performance}
\end{figure*}

\subsection{Results on Extremely Unbalanced Datasets} \label{sec:exp_results}
We generate the extremely unbalanced settings for both binary and multi-class classification from popular vision datasets CIFAR-10\cite{krizhevsky2009learning}, CIFAR-100\cite{krizhevsky2009learning}, PathMNIST\cite{yang2021medmnist} and SVHN\cite{netzer2011reading}. CIFAR-10 and SVHN both initially have 10 balanced classes while CIFAR-100 has $100$ balanced classes and PathMNIST has $9$ classes. In all of CIFAR-10, CIFAR-100 and SVHN, we construct the large ``other'' class by grouping the majority of the original classes into one out-of-distribution class. Suppose there are originally $M$ ($M=10$ or $1000$) balanced classes in the original dataset, we form a $K$ ($K \ll M$) class extremely unbalanced dataset by reusing classes $1, ..., K-1$ as in the original dataset, whereas class $K$ contains all examples in classes $K, ..., M$ in the original dataset. For PathMNIST, we consider the task of identifying cancer-associated stroma from the rest of hematoxylin \& eosin stained histological images. Table~\ref{tab:dataset} shows the detailed sizes of the extremely unbalanced datasets.
\begin{table}[H]
    \begin{center}
    \begin{small}
    \begin{sc}
    \begin{tabular}{ccccc}
    \toprule
    Name & \# Classes & $N_K$ & $\sum_{k=1}^{K-1} N_k$ & $\epsilon$  \\
    \midrule
    CIFAR-10 & $2$ & $45000$ & $5000$ & $.1111$\\
    CIFAR-10 & $3$ & $40000$ & $10000$ & $.1250$\\
    CIFAR-100 & $2$ & $49500$ & $500$ & $.0101$\\
    CIFAR-100 & $3$ & $49000$ & $1000$ & $.0102$\\
    CIFAR-100 & $10$ & $40500$ & $9500$ & $.0123$\\
    SVHN & $2$ & $68309$ & $4948$ & $.0724$\\
    SVHN & $3$ & $54448$ & $18809$ & $.2546$\\
    PathMNIST & $2$ & $80595$ & $9401$ & $.1166$\\
    \bottomrule
    \end{tabular}
    \end{sc}
    \end{small}
    \end{center}
    \caption{Dataset details for each extremely unbalanced scenario. $N_K$ denotes the number of images in the out-of-distribution class while $\sum_{k=1}^{K-1} N_k$ is the total number of images in all in-distribution classes. $\epsilon$ is the class imbalance factor defined in Section~\ref{sec:notation}.}
    \label{tab:dataset}
\end{table}

\textbf{Comparison Algorithms:} 
We compare our algorithm \textbf{GALAXY} against eight baselines. \textbf{SIMILAR} \citep{kothawade2021similar}, \textbf{Cluster Margin} \citep{citovsky2021batch}, \textbf{BASE} \citep{emam2021active}, \textbf{BADGE} \citep{ash2019deep} and \textbf{BAIT} \citep{ash2021gone} have all been described in Section~\ref{sec:related}. For \textbf{SIMILAR}, we use the FLQMI relaxation of the submodular mutual information (SMI). We are unable to compare to the FLCMI relaxation of the submodular conditional mutual information (SCMI) due to excessively high memory usage required by the submodular maximization at pool size $N=50000$. As demonstrated in \citet{kothawade2021similar} however, one should expect only marginal improvement over FLQMI relaxation of the SMI. For \textbf{Cluster Margin} we choose clustering hyperparameters so there are exactly $50$ clusters. We choose margin batch size to be $k_m = 125$ while the target batch size is set to $k_t = B = 100$.

In addition to the above methods, \textbf{Confidence Sampling} \citep{settles2009active} is a type of uncertainty sampling that queries the least confident examples in terms of $\max_{k \in [K]} [f_\theta(x)]_k$. Here, $f_\theta$ is a classifier that outputs softmax scores and maximization is take with respect to classes. 
\textbf{Most Likely Positive} \citep{jiang2018efficient,warmuth2001active,warmuth2003active} is a heuristic often used in active search, where the algorithm selects the examples most likely to be in the in-distribution classes by its predictive probabilities.
Lastly, \textbf{Random} is the naive uniform random strategy. For each setting, we average over $4$ individual runs for each of \textbf{GALAXY}, \textbf{Cluster Margin}, \textbf{BASE}, \textbf{Confidence Sampling}, \textbf{Most Likely Positive} and \textbf{Random}. Due to computational constraints, we are only able to have single runs for each of \textbf{SIMILAR}, \textbf{BADGE} and \textbf{BAIT}. For algorithms with multiple runs, the standard error is also plotted as the confidence intervals. To demonstrate the active gains more clearly, all of our curves are smoothed by moving average with window size $10$.

As shown in Figure~\ref{fig:performance}, to achieve any balanced accuracy, \texttt{GALAXY} outperforms all baselines in terms of the number of labels requested, saving up to $30\%$ queries in some cases when comparing to the second best method. For example in unbalanced SVHN with 2 classes, to achieve $92\% accuracy$, \texttt{GALAXY} takes $1700$ queries while the second best algorithm takes $2500$ queries. In unbalanced CIFAR-100 with 3 classes, to reach $66\%$ accuracy, GALAX takes $1600$ queries while the second best algorithm takes $2200$ queries. As expected, \textbf{Cluster Margin} and \textbf{BASE} are competitive in many settings as they also target unbalanced settings. \textbf{BAIT} and \textbf{BADGE} tend to perform less well primarily due to their focus on collecting data-diverse examples, which has roughly the same class-imbalance as the pool. Full experimental results on all $8$ settings are presented in Appendix~\ref{sec:full_result}. In Appendix~\ref{sec:full_result}, we also include an experiment on CIFAR-100, 10 classes with batch size $1000$ showing the superiority of our method in the large budget regime.

\begin{figure}
    \centering
    \includegraphics[width=.6\linewidth]{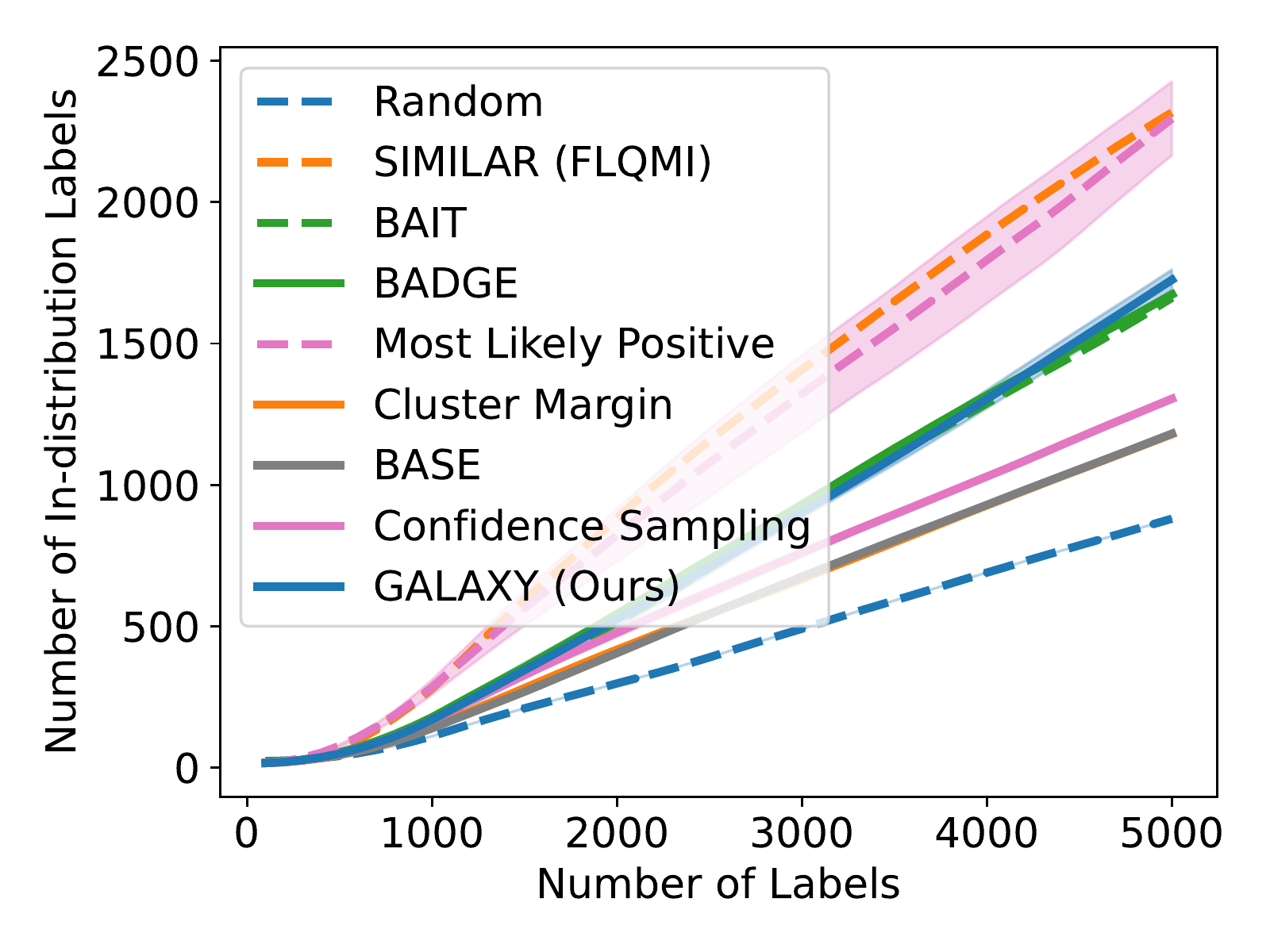}
    \caption{Number of in-distribution labels for CIFAR-10, 3 classes}
    \label{fig:minority_label}
\end{figure}
As shown in Figure~\ref{fig:minority_label}, \texttt{GALAXY}'s success relies on its inherent feature of collecting a more balanced labeled set of uncertain examples. In particular, \texttt{GALAXY} is collecting a significantly more in-distribution examples than most baseline algorithms including uncertainty sampling. On the other hand, although \textbf{SIMILAR} and \textbf{Most Likely Positive} both collect more examples in the in-distribution classes, their inferiority in balanced accuracy suggests that the examples are not representative enough. Indeed, both methods are inherently collecting labels for example that are \emph{certain}. This thus suggests the importance of collecting batches that are not only balanced but also uncertain.

\subsection{Comparison: $S^2$ vs \texttt{GALAXY}} \label{sec:exp_compare}

\begin{figure}
    \centering
    \includegraphics[width=.6\linewidth]{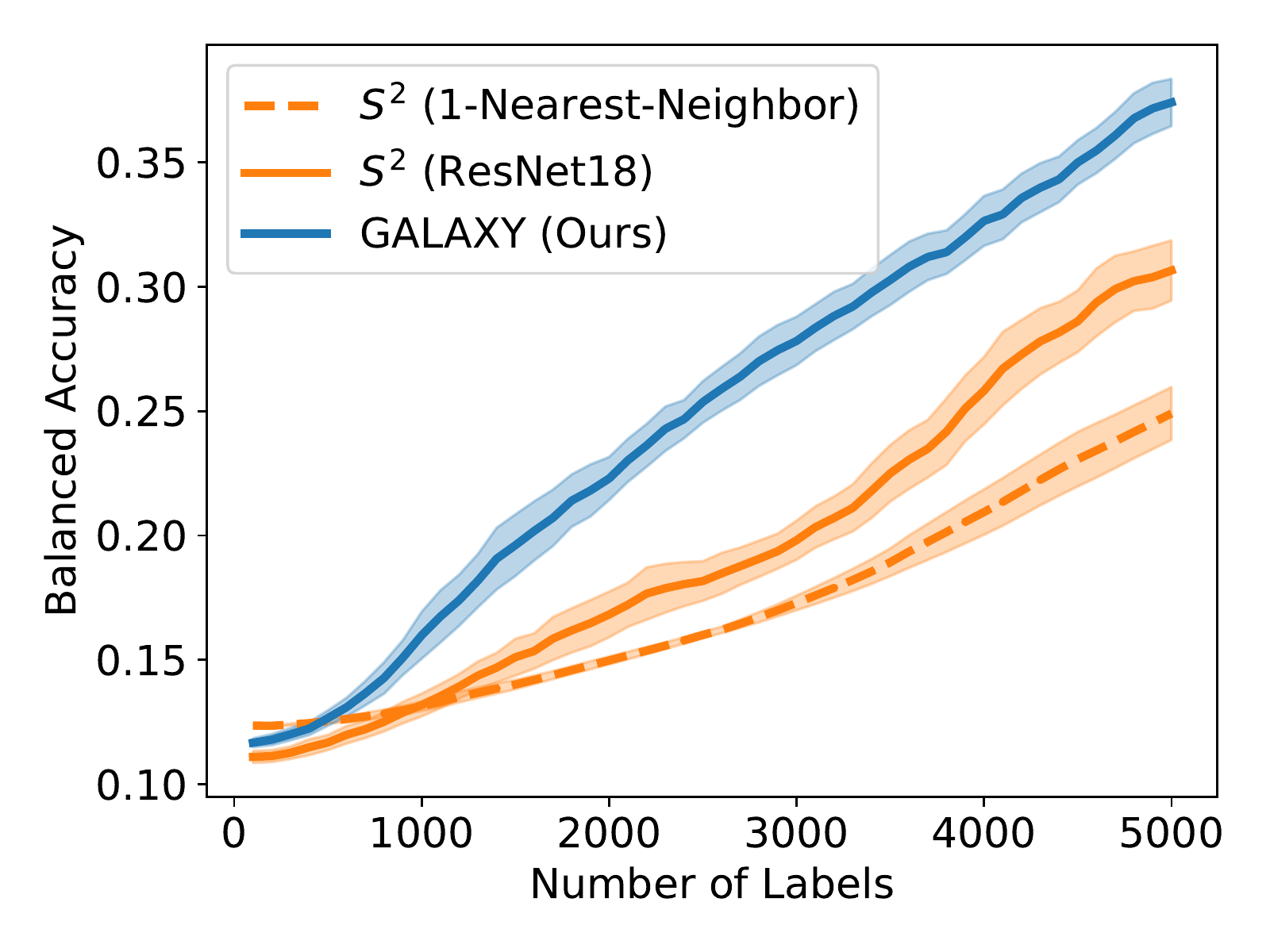}
    \caption{Comparison of \texttt{GALAXY} with vanilla $S^2$ with 1-nearest-neighbor and neural network classifiers. We use the CIFAR-100, 10 classes data setting for comparison.}
    \label{fig:s2_comparison}
\end{figure}
In this section, we conduct experiment to compare the original $S^2$ approach \citep{dasarathy2015s2} against our method. For $S^2$, we construct a 10-nearest-neighbor graph from feature vectors of a ResNet-18 model pretrained on ImageNet. We show two curves of $S^2$ using two different models -- 1-nearest-neighbor prediction on the graph and neural network training in Section~\ref{sec:exp_setup}. We note that the models training does not affect the $S^2$ active queries, whereas \texttt{GALAXY} constantly constructs graphs based on these updated models. As shown in Figure~\ref{fig:s2_comparison}, \texttt{GALAXY} outperforms $S^2$ with both models by a significant margin, showing the necessity on learning and constructing better graphs (Algorithm~\ref{alg:build_graph}).

\section{Future Direction}
In this paper, we propose a novel graph-based approach to deep active learning that particularly targets the extreme class imbalance cases. We show that our algorithm \texttt{GALAXY} outperforms all existing methods by collecting a mixture of balanced yet uncertain examples.
\texttt{GALAXY} runs on similar time complexity as other uncertainty based methods by retraining the neural network model only after each batch. However, it still requires sequential and synchronous labelling within each batch. This means the human labelling effort cannot be parallelized by multiple annotators. For future work, we would like to incorporate asynchronous labelling and investigate its effect on our algorithm.

\section*{Acknowledgement}
We thank Andrew Wagenmaker for insightful discussions. This work has been supported in part by NSF Award 2112471.
\bibliography{reference}
\bibliographystyle{icml2022}

\newpage
\appendix
\onecolumn

\section{Proof of Theorem~\ref{thm:bisection}} \label{sec:bisection_proof}
\begin{proof}
First, when $n_{\texttt{ID}} + n_{\texttt{OOD}} \geq 2^{z+2} - 1$, it's easy to see by induction that after $k \leq z$ queries, the region of uncertainty shrinks to have at least $2^{z+2 - k} - 1$ examples. Therefore, after $z$ steps, we must have $n' \geq 3$.

Next, let $m'_{\texttt{OOD}}$ denote the number of \texttt{OOD} labels queried after bisecting $z$ steps, namely $m'_{\texttt{OOD}} + z = m_{\texttt{OOD}}$. Since in the last $n'$ examples, the number of \texttt{ID} examples is $n_{\texttt{ID}} \sim \text{Unif}(\{1, ..., n' - 1\})$, we must have the number of \texttt{OOD} examples to be symmetrically $n' - n_{\texttt{ID}} \sim \text{Unif}(\{1, ..., n' - 1\})$. Therefore, due to symmetry of distribution and the bisection procedure, in expectation the bisection procedure queries equal numbers of \texttt{ID} and \texttt{OOD} examples, i.e. $\E[m_{\texttt{ID}}] = \E[m'_{\texttt{OOD}}]$. Together we must have
\begin{align*}
    \frac{\E[m_{\texttt{ID}}]}{\E[m_{\texttt{OOD}}]} = \frac{\E[m_{\texttt{ID}}]}{z + \E[m_{\texttt{ID}}]} \geq \frac{\frac{1}{2} \log_2 (n')}{z + \frac{1}{2}\log_2 (n')}
\end{align*}
where the last inequality follows from the total number of queries $m_{\texttt{ID}} + m'_{\texttt{OOD}} \geq \log_2 (n')$ so $\E[m_{\texttt{ID}}] \geq \frac{1}{2} \log_2 (n')$.
\end{proof}

\section{Proof of Corollary~\ref{thm:galaxy_balance}}
\begin{proof} \label{sec:galaxy_proof}
As shown in Theorem~\ref{thm:bisection}, even without the $B'$ additional queries, we must have $\E[m_{\texttt{ID}}] \geq \frac{1}{2}\log_2 (n')$. Now, for the process of querying two sides of the cut, with $B'$ queries we can guarantee that at least $\min\{n_{\texttt{ID}}, \lfloor \frac{B'}{2} \rfloor\}$ examples to the left of the cut must have been queried and are in \texttt{ID}. Therefore, $\E[m_{\texttt{ID}}] \geq \E[\min\{n_{\texttt{ID}}, \lfloor \frac{B'}{2} \rfloor\}] \geq \lfloor \frac{B'}{4} \rfloor$. As a result, we the have $\E[m_{\texttt{ID}}] \geq y$ and $\E[m_{\texttt{OOD}}] \leq B - y$, so
\begin{align*}
    \frac{\E[m_{\texttt{ID}}]}{\E[m_{\texttt{OOD}}]} &\geq \frac{y}{B - y} = \frac{y}{z + B' + \lfloor \log_2(n') \rfloor - y} \\
    &\geq \frac{y}{z + (4y + 3) + 2y - y} = \frac{y}{z + 5y + 3}
\end{align*}
\end{proof}

\section{Proof of Theorem~\ref{thm:galaxy_noise}} \label{sec:noise_proof}
\begin{lemma}[Noise Tolerance of Bisection] \label{lem:noise}
Let $n = n_{\texttt{ID}} + n_{\texttt{OOD}}$. If the true label of each example in the region of uncertainty is corrupted independently with probability $\frac{\delta}{\lceil \log_2 n \rceil}$, the bisection procedure recovers the true uncertainty threshold with probability at least $1 - \delta$.
\end{lemma}
\begin{proof}
Bisection procedure will make $\lceil \log_2 n \rceil$ queries and for each query the label could be corrupted with probability $\frac{\delta}{\lceil \log_2 n \rceil}$. Therefore, by union bound, we must then have
\begin{align*}
    \P(\text{\#corrupt queries} > 0) \leq \lceil \log_2 n \rceil \cdot \frac{\delta}{\lceil \log_2 n \rceil} = \delta.
\end{align*}
\end{proof}

Now we start to prove Theorem~\ref{thm:galaxy_noise}.
\begin{proof}
By Lemma~\ref{lem:noise}, we know with probability $1-\delta$, all of the bisection queries are not corrupted. Furthermore, as proved in Theorem~\ref{thm:bisection}, we at least take $\E[m_{\texttt{ID}}] \geq \log_2 n$ number of queries in class \texttt{ID}, so $\E[m_{\texttt{OOD}}] \leq B - \log_2 n$. As a result, with probability at least $1 - \delta$ we have the desired balancedness bound.
\end{proof}

\section{Full Experimental Results on CIFAR-10, CIFAR-100 and SVHN} \label{sec:full_result}
\begin{figure*}[t!]
    \centering
    \begin{subfigure}[t]{.49\textwidth}
        \centering
        \includegraphics[width=\textwidth]{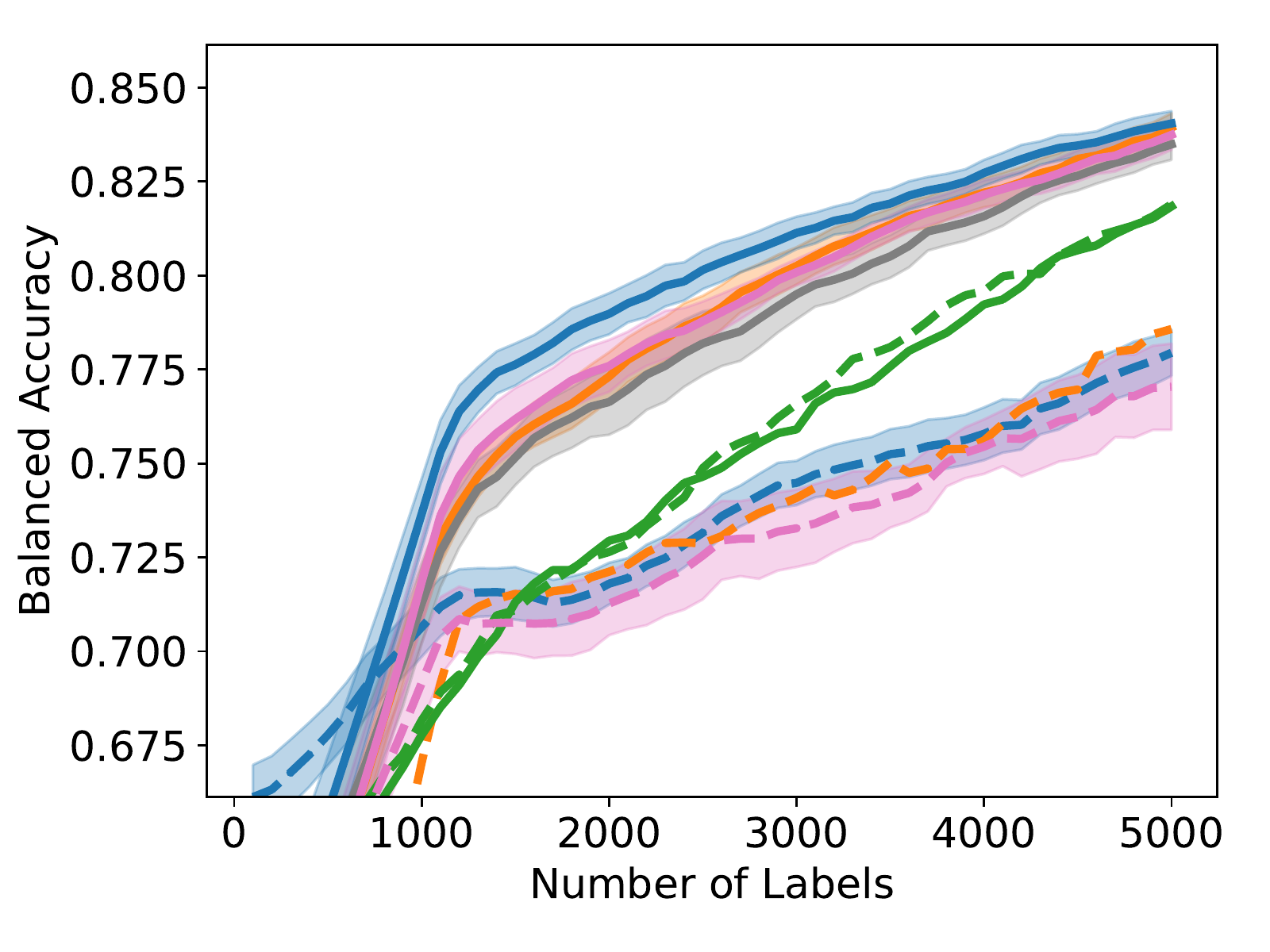}
        \caption{$ACC_{bal}$}
    \end{subfigure}
    \begin{subfigure}[t]{.49\textwidth}
        \centering
        \includegraphics[width=\textwidth]{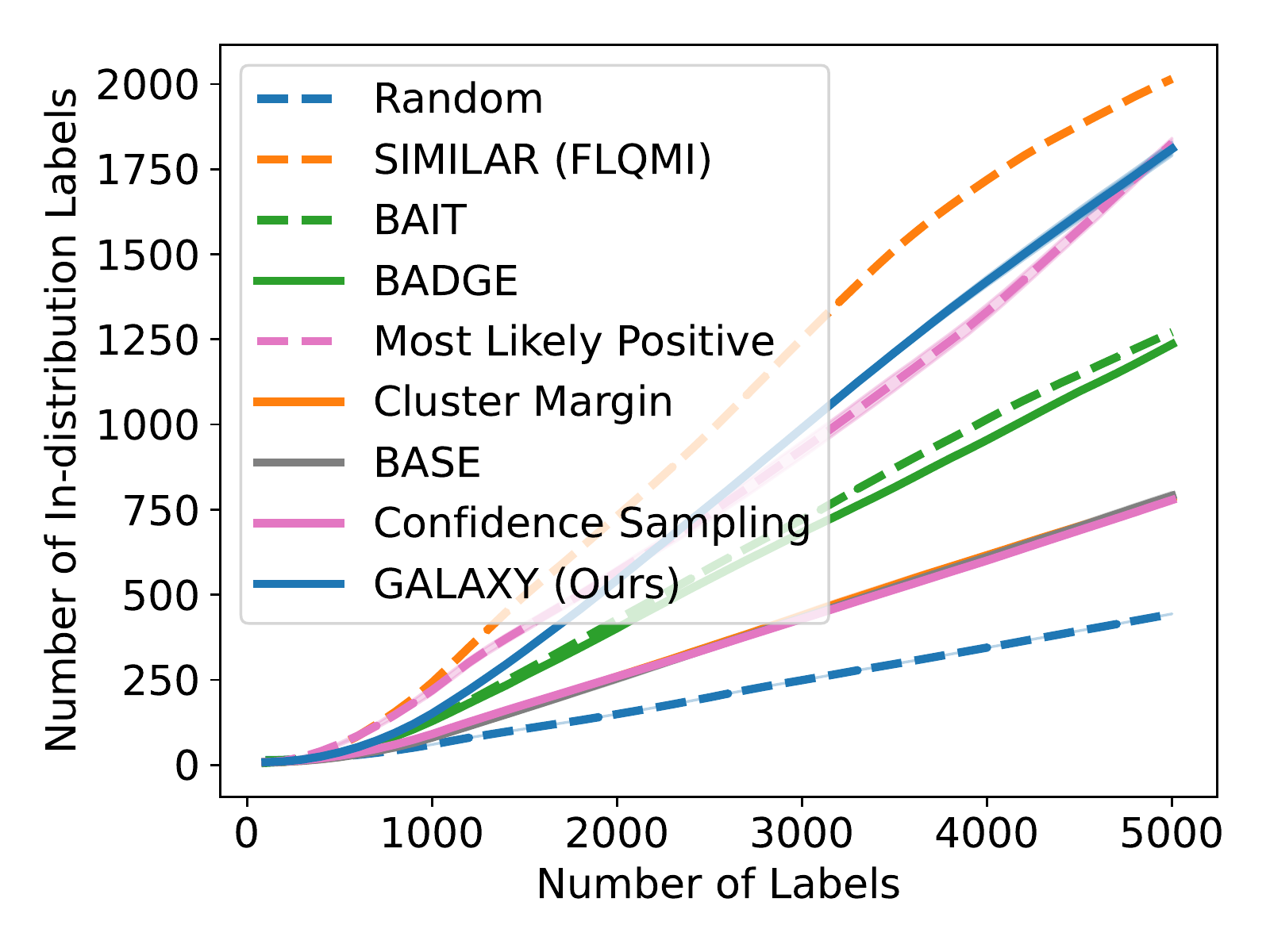}
        \caption{\#In-distribution Label}
    \end{subfigure}
    \caption{CIFAR-10, 2 classes}
\end{figure*}
\begin{figure*}[t!]
    \centering
    \begin{subfigure}[t]{.49\textwidth}
        \centering
        \includegraphics[width=\textwidth]{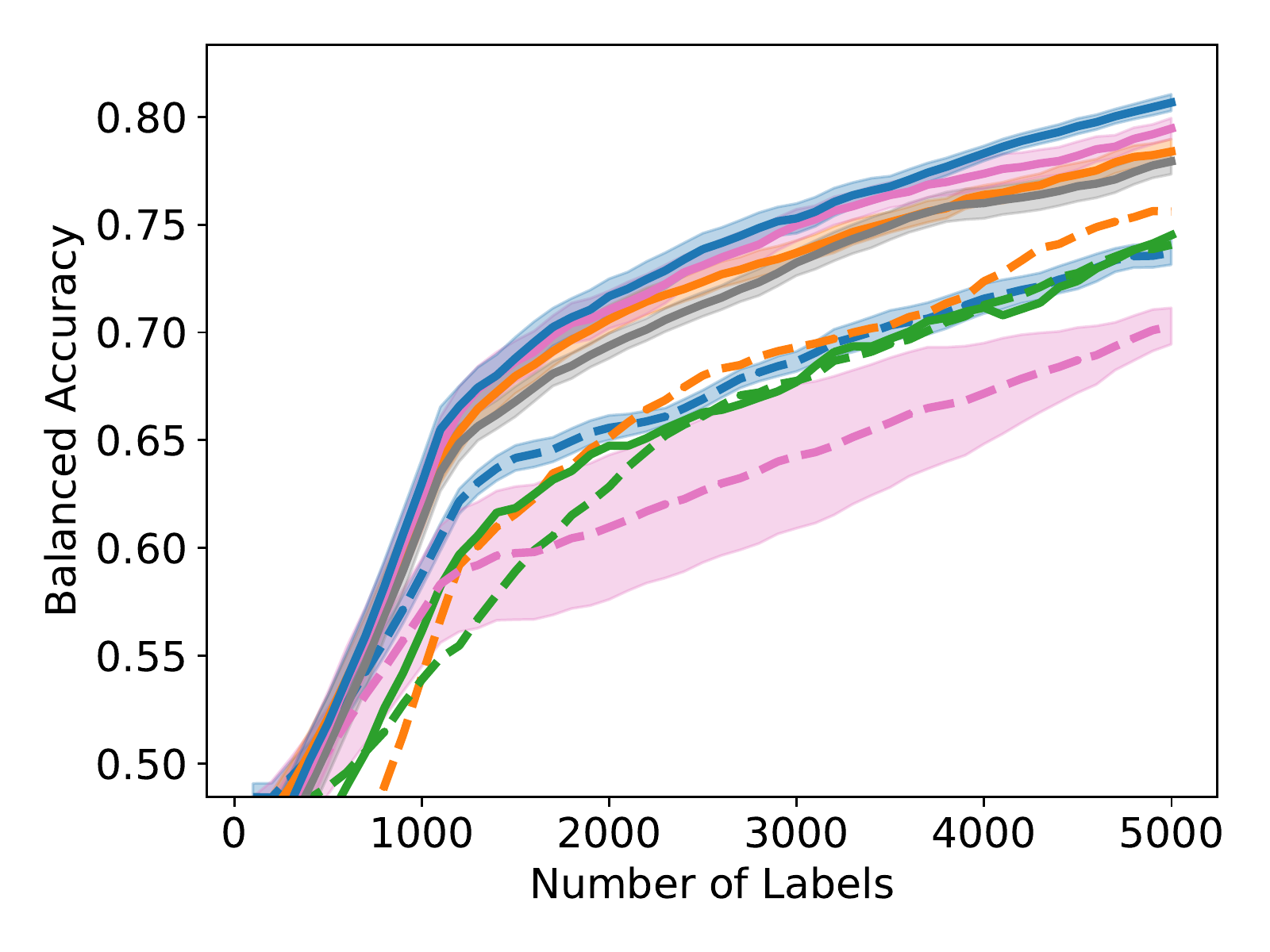}
        \caption{$ACC_{bal}$}
    \end{subfigure}
    \begin{subfigure}[t]{.49\textwidth}
        \centering
        \includegraphics[width=\textwidth]{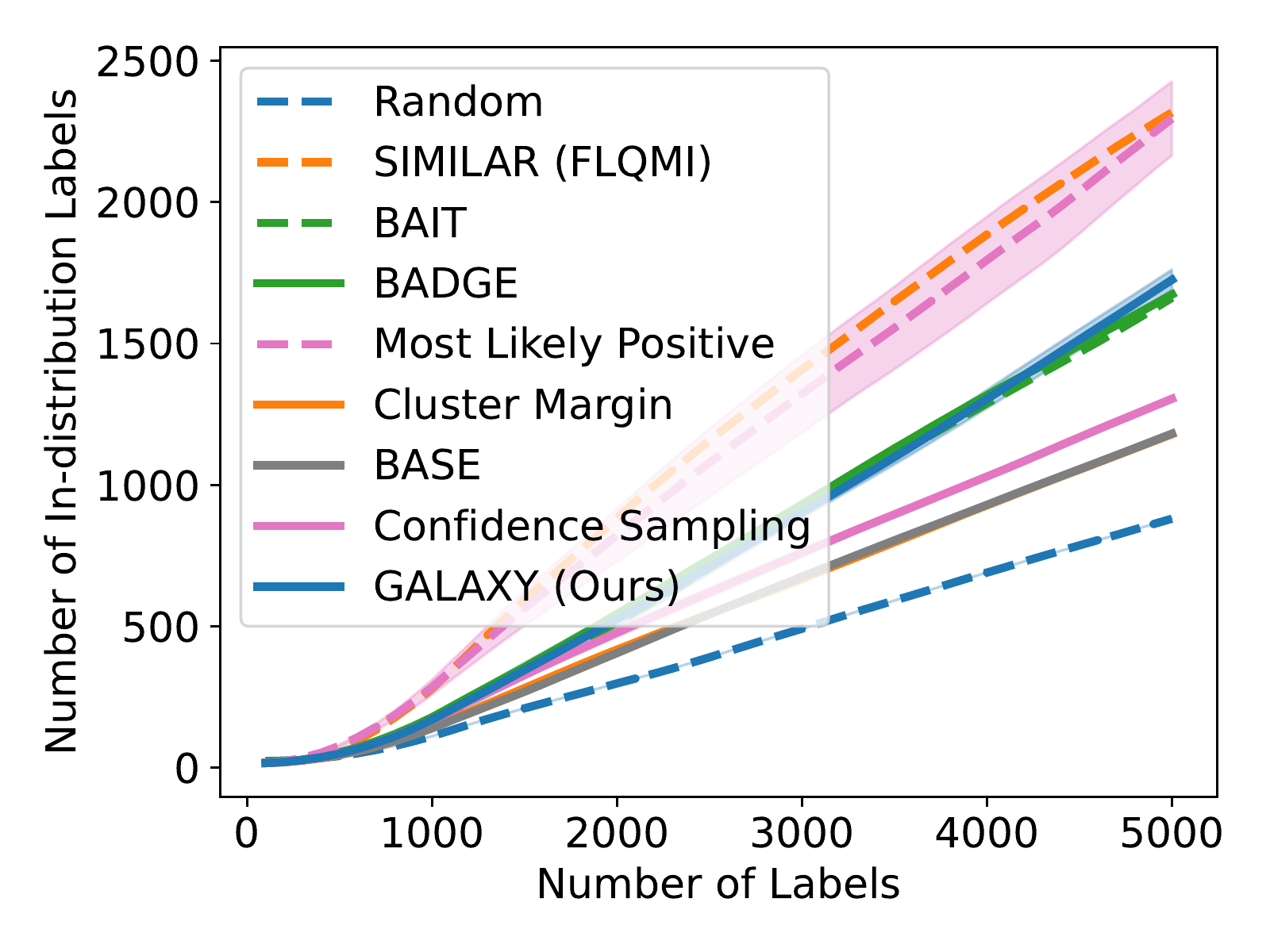}
        \caption{\#In-distribution Label}
    \end{subfigure}
    \caption{CIFAR-10, 3 classes}
\end{figure*}
\begin{figure*}[t!]
    \centering
    \begin{subfigure}[t]{.49\textwidth}
        \centering
        \includegraphics[width=\textwidth]{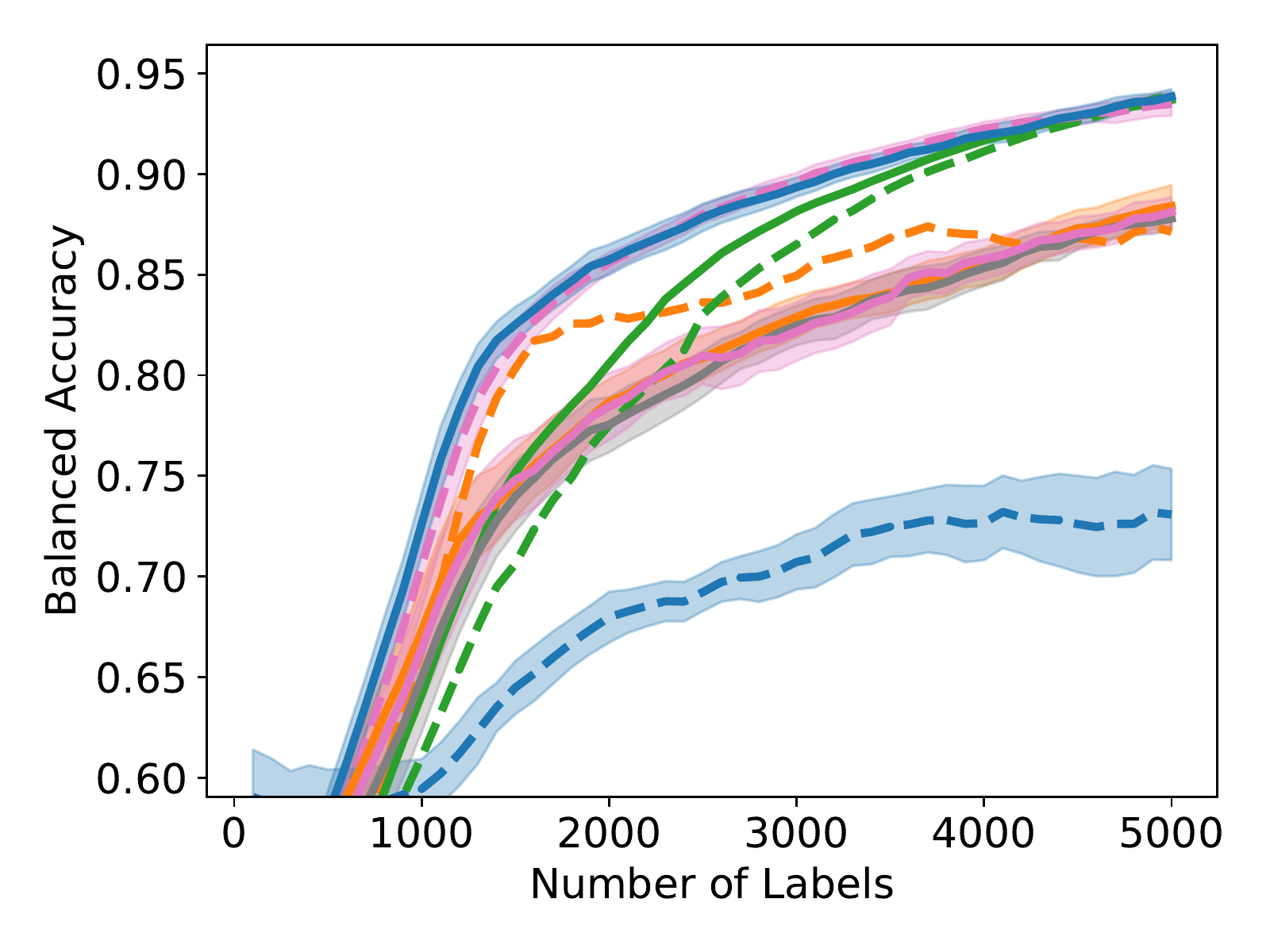}
        \caption{$ACC_{bal}$}
    \end{subfigure}
    \begin{subfigure}[t]{.49\textwidth}
        \centering
        \includegraphics[width=\textwidth]{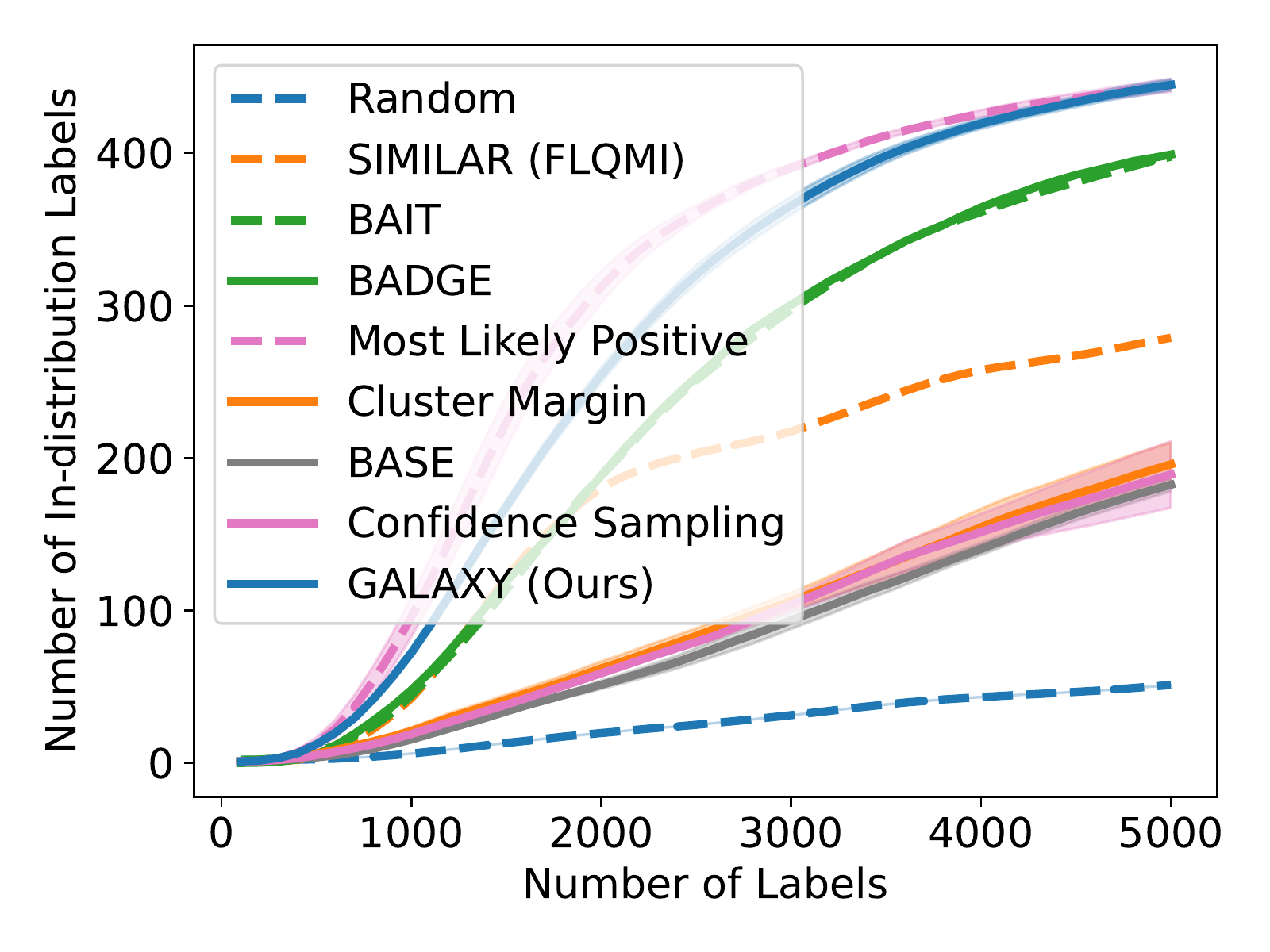}
        \caption{\#In-distribution Label}
    \end{subfigure}
    \caption{CIFAR-100, 2 classes}
\end{figure*}
\begin{figure*}[t!]
    \centering
    \begin{subfigure}[t]{.49\textwidth}
        \centering
        \includegraphics[width=\textwidth]{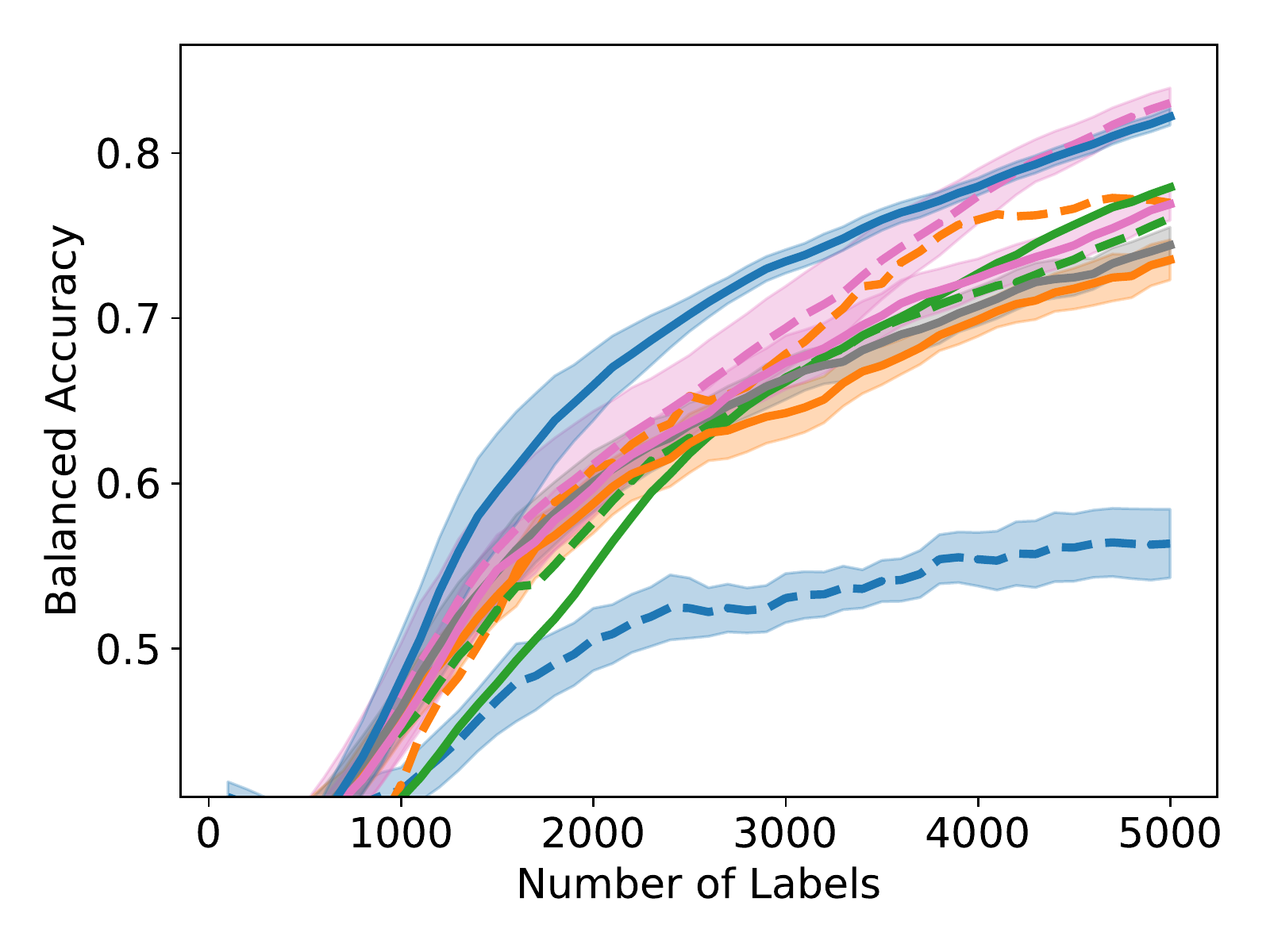}
        \caption{$ACC_{bal}$}
    \end{subfigure}
    \begin{subfigure}[t]{.49\textwidth}
        \centering
        \includegraphics[width=\textwidth]{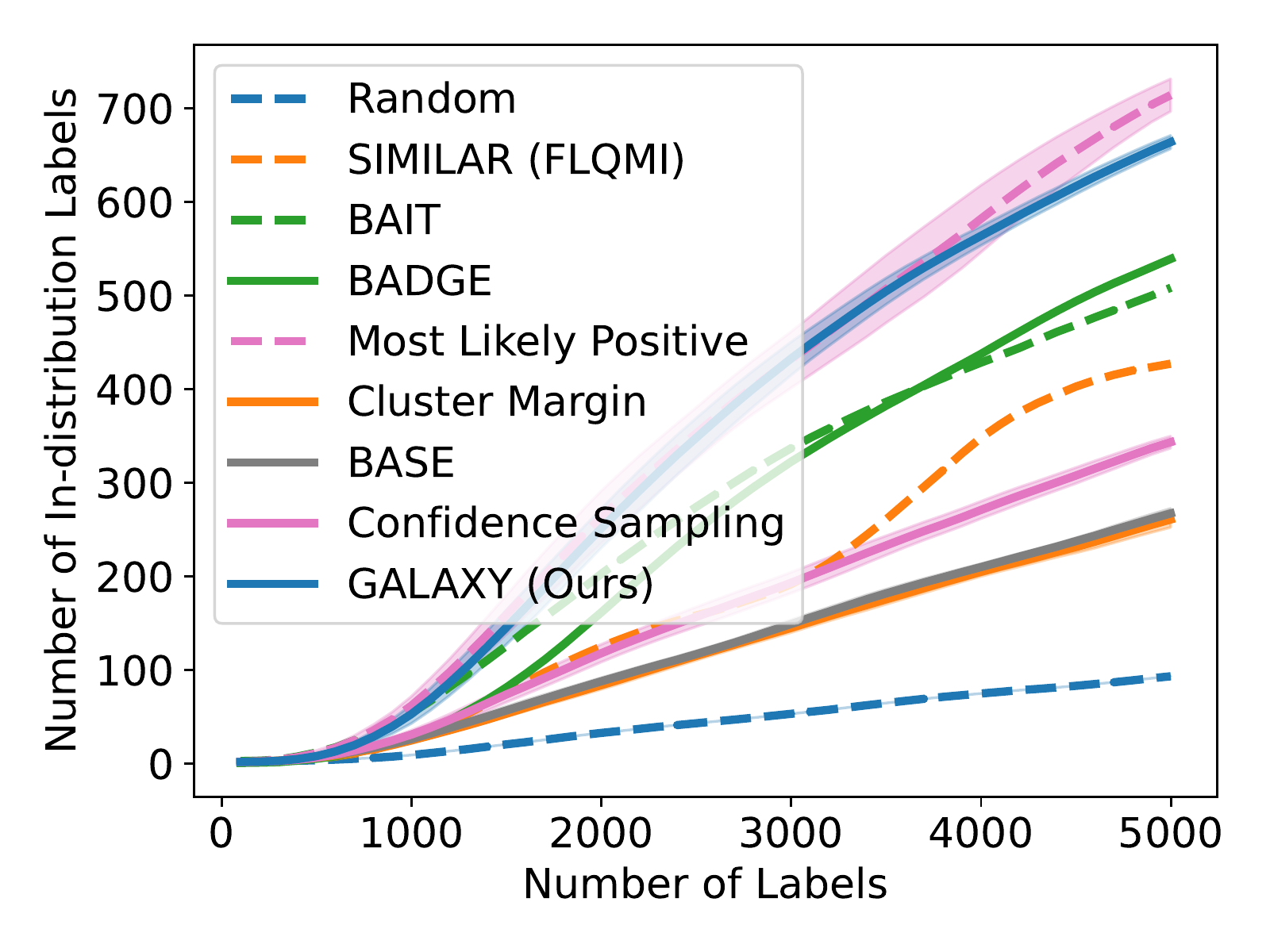}
        \caption{\#In-distribution Label}
    \end{subfigure}
    \caption{CIFAR-100, 3 classes}
\end{figure*}
\begin{figure*}[t!]
    \centering
    \begin{subfigure}[t]{.49\textwidth}
        \centering
        \includegraphics[width=\textwidth]{figure/cifar100_unbalanced_10_accuracy.pdf}
        \caption{$ACC_{bal}$}
    \end{subfigure}
    \begin{subfigure}[t]{.49\textwidth}
        \centering
        \includegraphics[width=\textwidth]{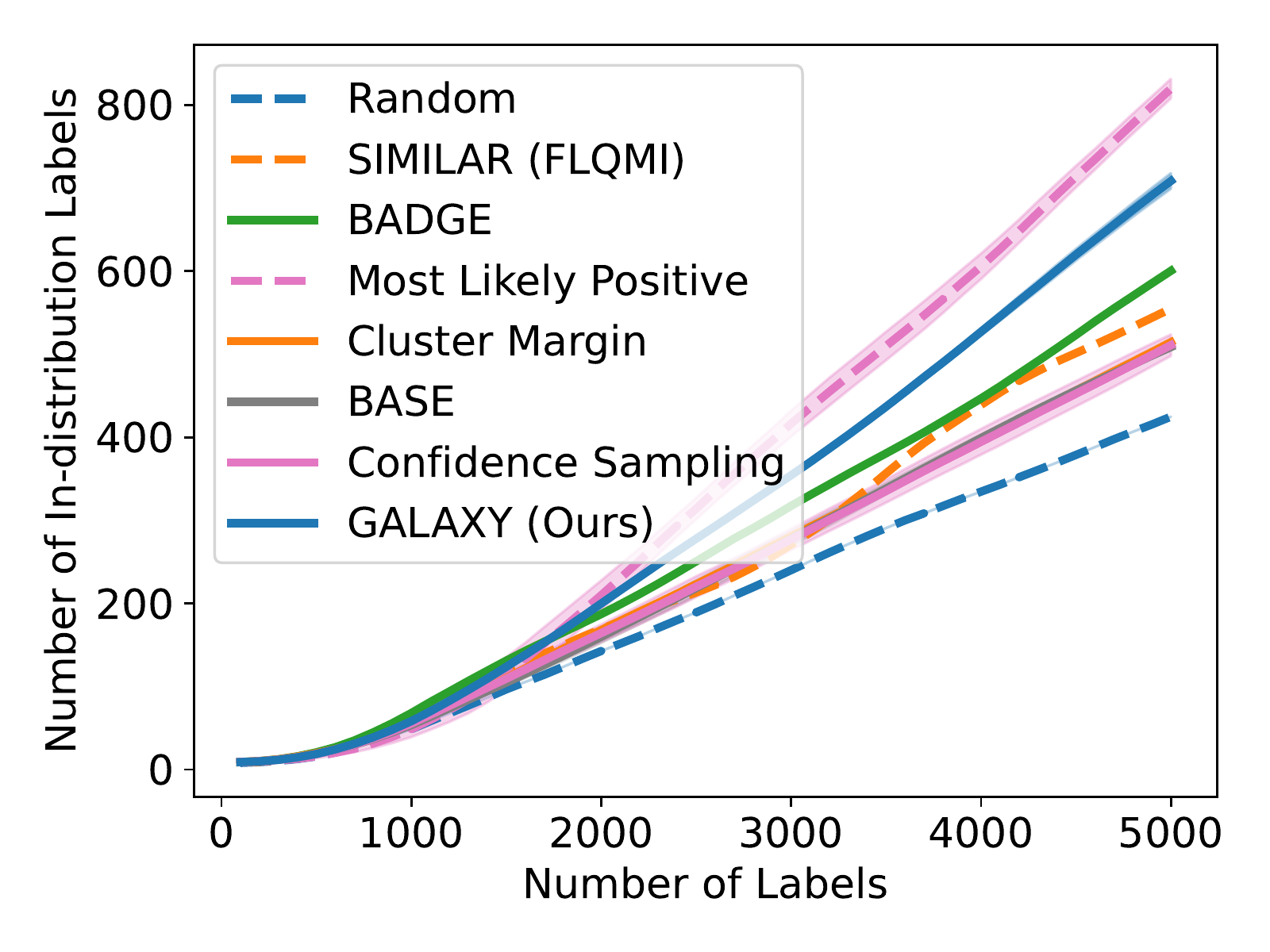}
        \caption{\#In-distribution Label}
    \end{subfigure}
    \caption{CIFAR-100, 10 classes}
\end{figure*}
\begin{figure*}[t!]
    \centering
    \begin{subfigure}[t]{.49\textwidth}
        \centering
        \includegraphics[width=\textwidth]{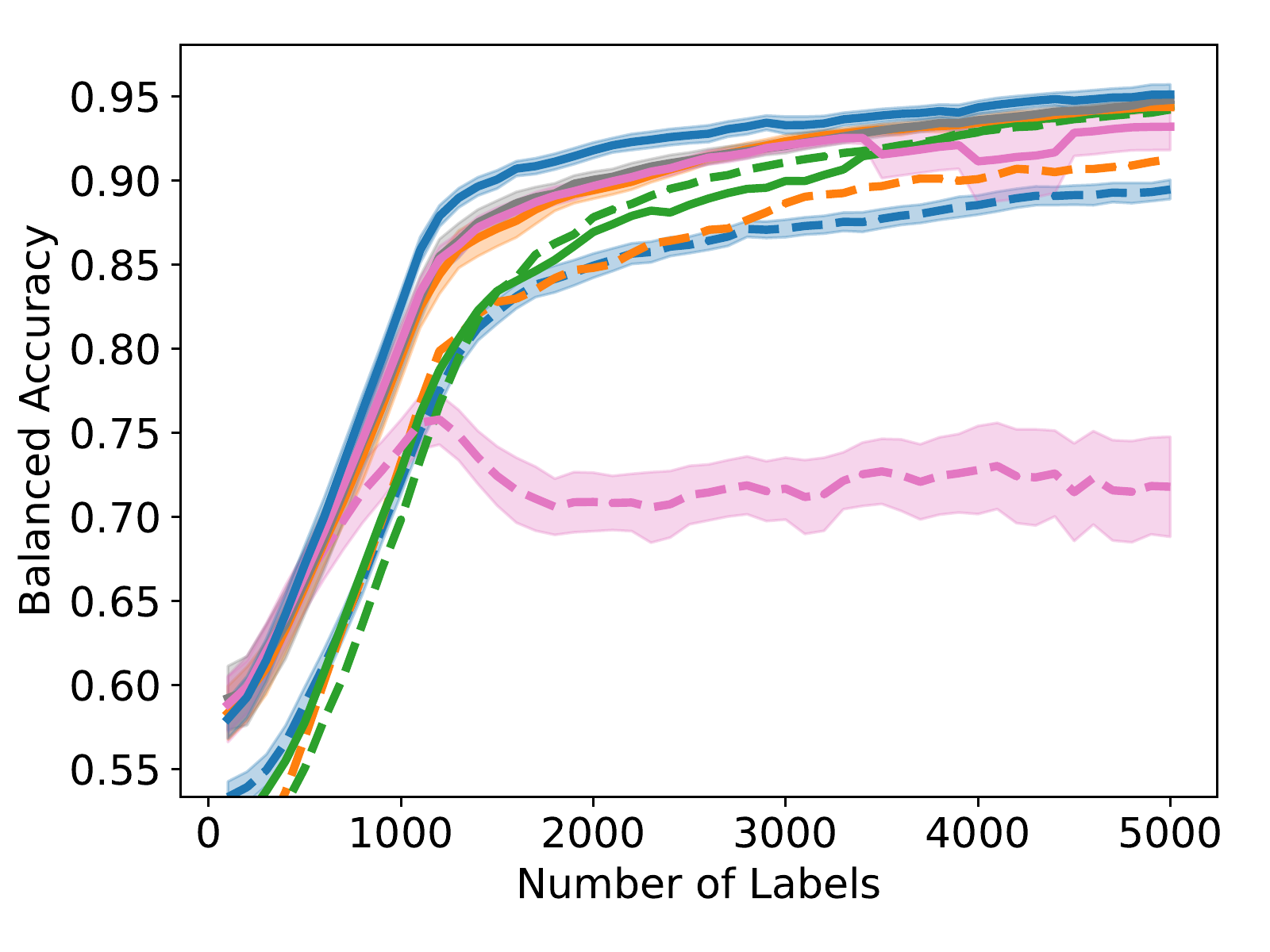}
        \caption{$ACC_{bal}$}
    \end{subfigure}
    \begin{subfigure}[t]{.49\textwidth}
        \centering
        \includegraphics[width=\textwidth]{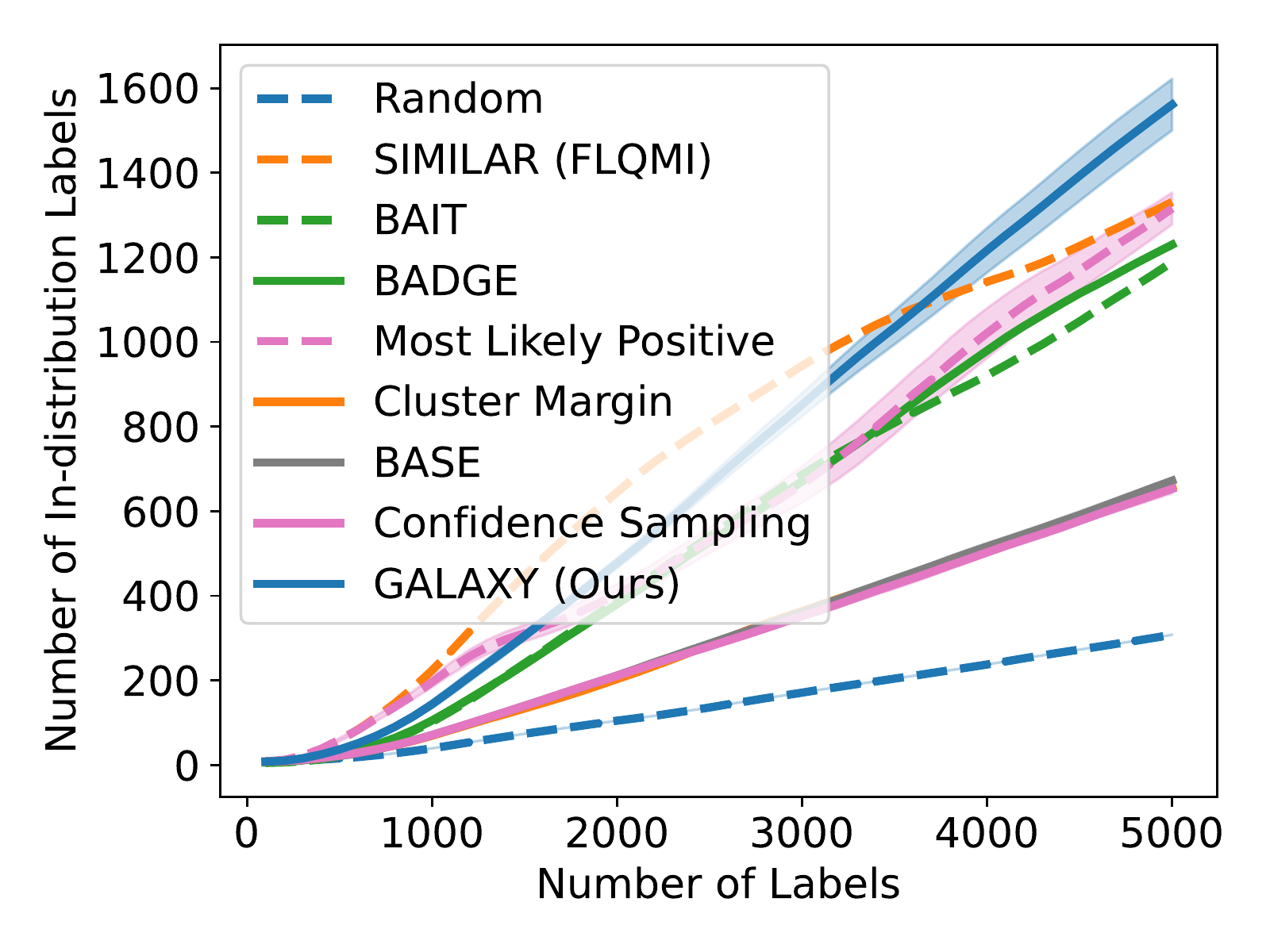}
        \caption{\#In-distribution Label}
    \end{subfigure}
    \caption{SVHN, 2 classes}
\end{figure*}
\begin{figure*}[t!]
    \centering
    \begin{subfigure}[t]{.49\textwidth}
        \centering
        \includegraphics[width=\textwidth]{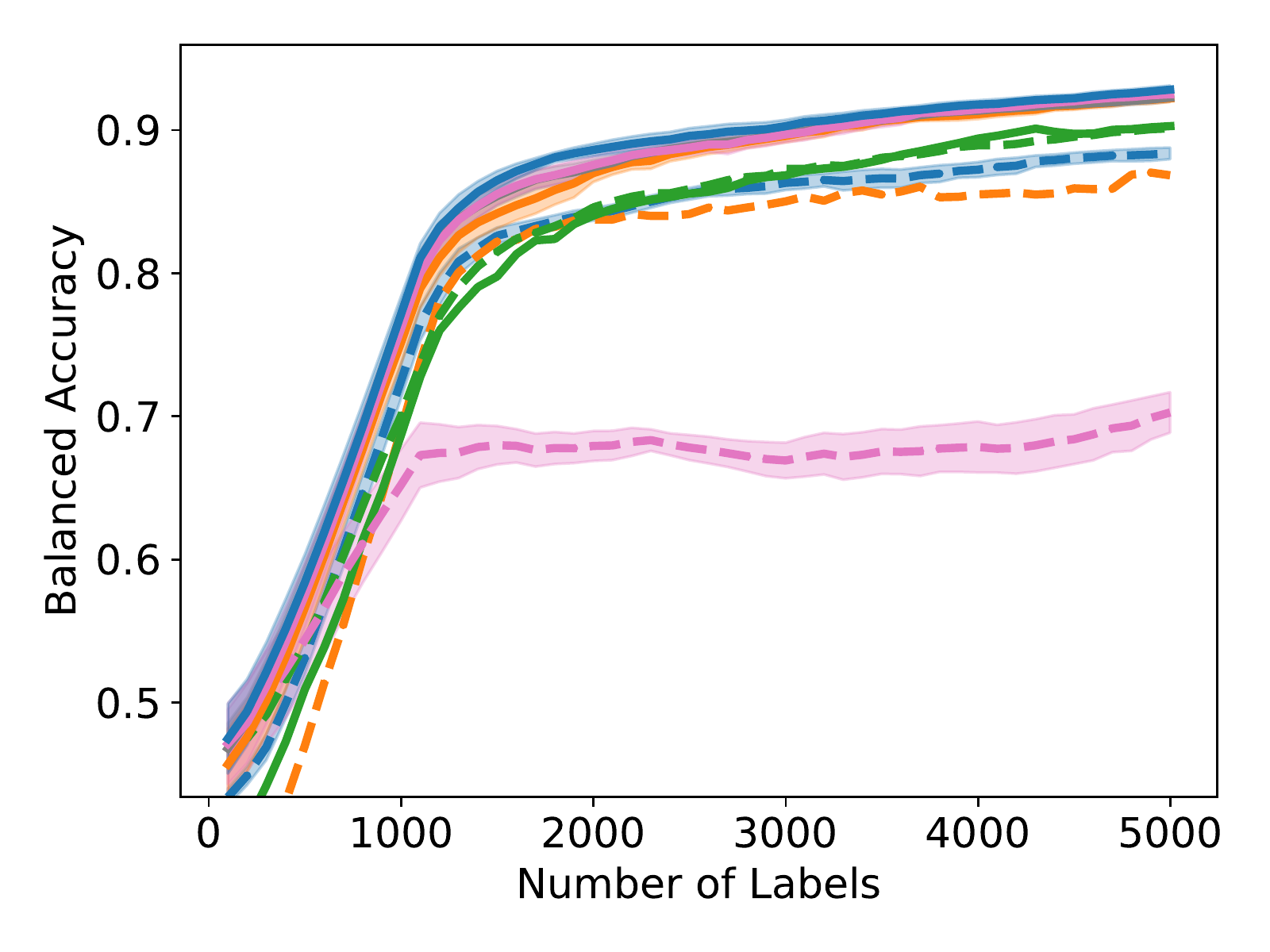}
        \caption{$ACC_{bal}$}
    \end{subfigure}
    \begin{subfigure}[t]{.49\textwidth}
        \centering
        \includegraphics[width=\textwidth]{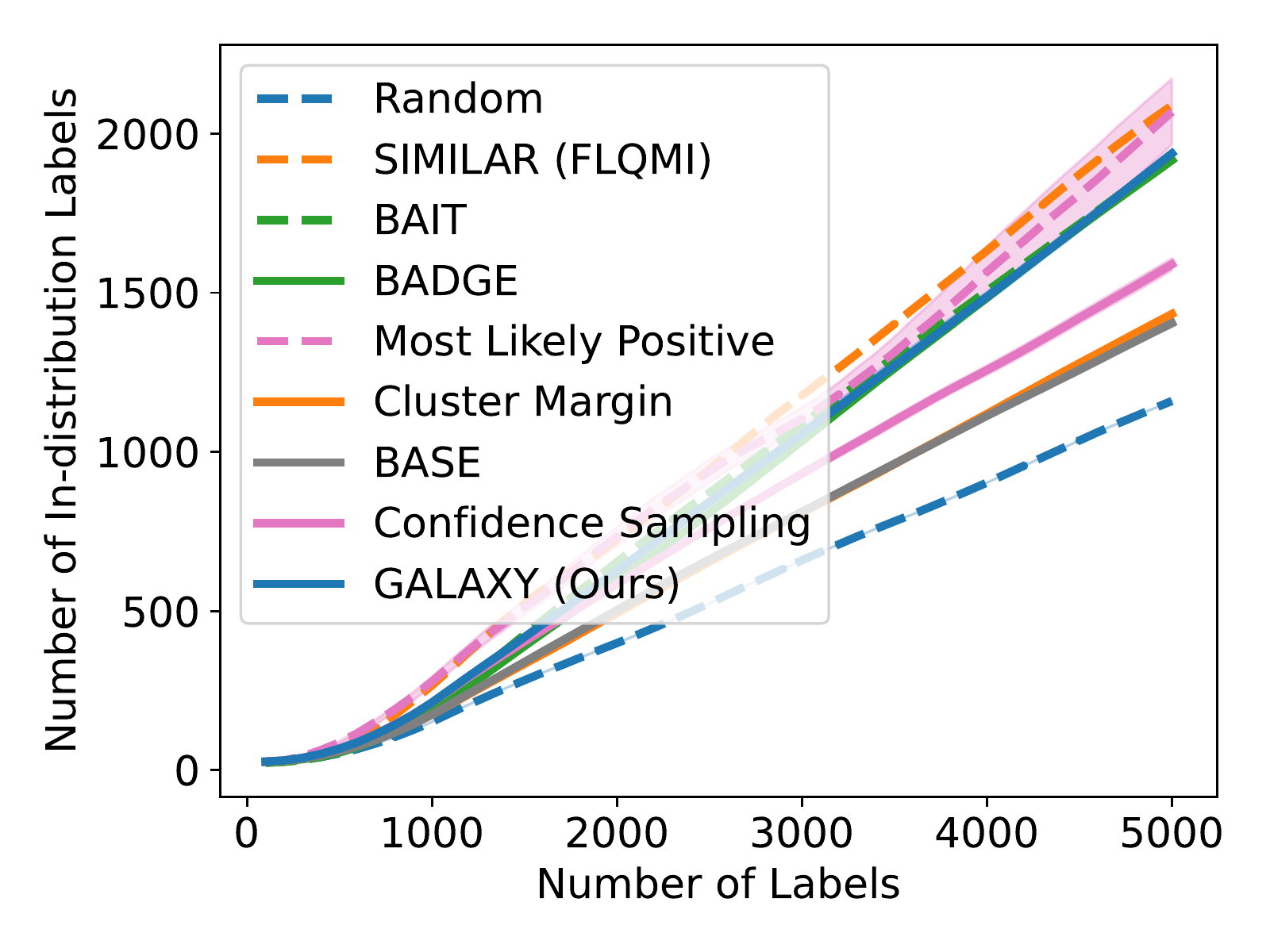}
        \caption{\#In-distribution Label}
    \end{subfigure}
    \caption{SVHN, 3 classes}
\end{figure*}

\begin{figure*}[t!]
    \centering
        \includegraphics[width=.5\textwidth]{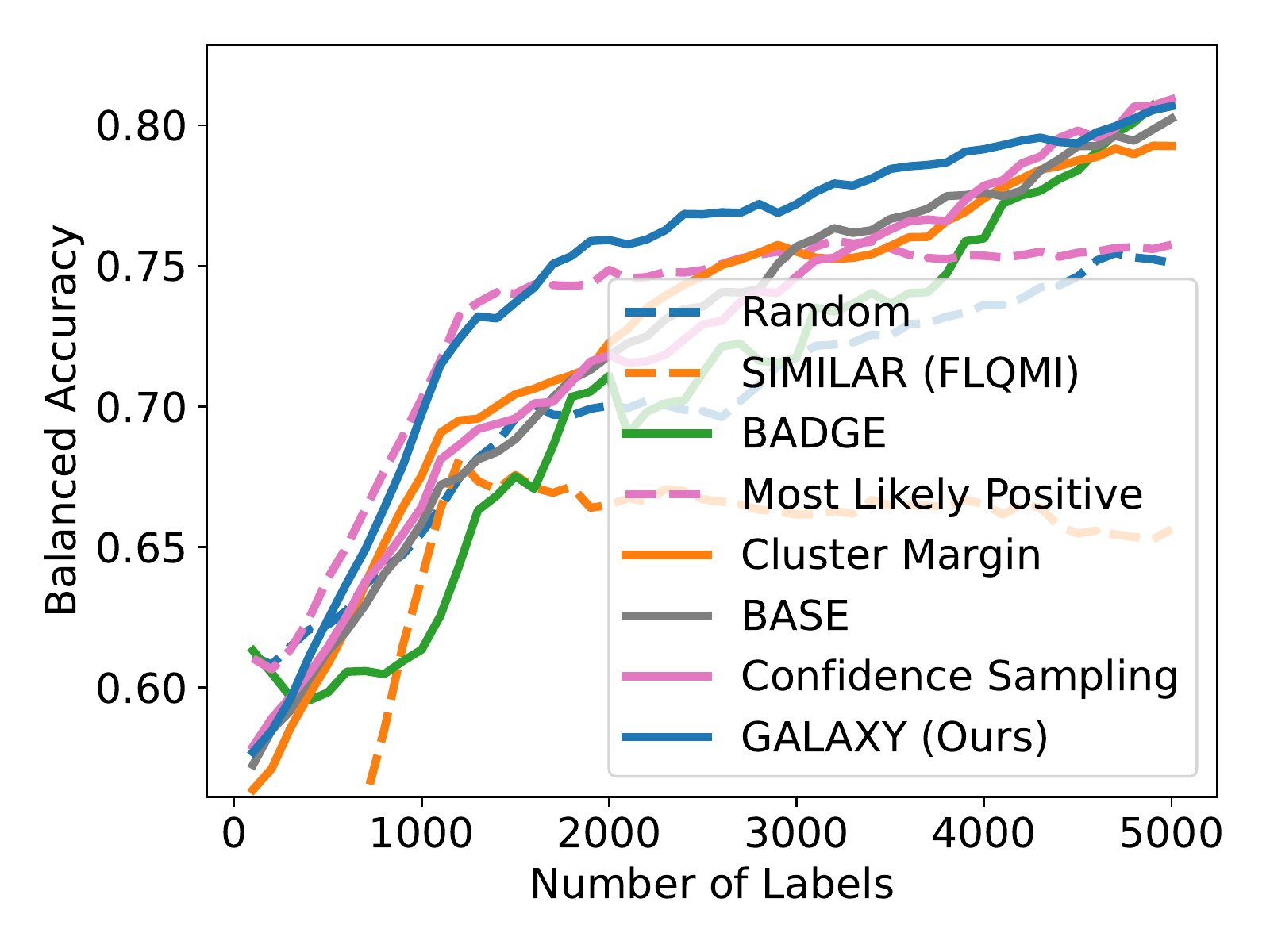}
    \caption{$ACC_{bal}$ for PathMNIST, $2$ classes. We only conduct $1$ run for each algorithm in this setting.}
\end{figure*}

\subsection{Large-budget Regime}
In Figure~\ref{fig:large_budget}, we use a batch size $1000$ and average over $3$ runs. We use a labelling budget of $30000$ out of the pool of size $50000$. Note that confidence sampling performs competitive in this case but could fail catastrophically in cases such as SVHN, 3 classes. 
\begin{figure}
    \centering
    \includegraphics[width=.5\linewidth]{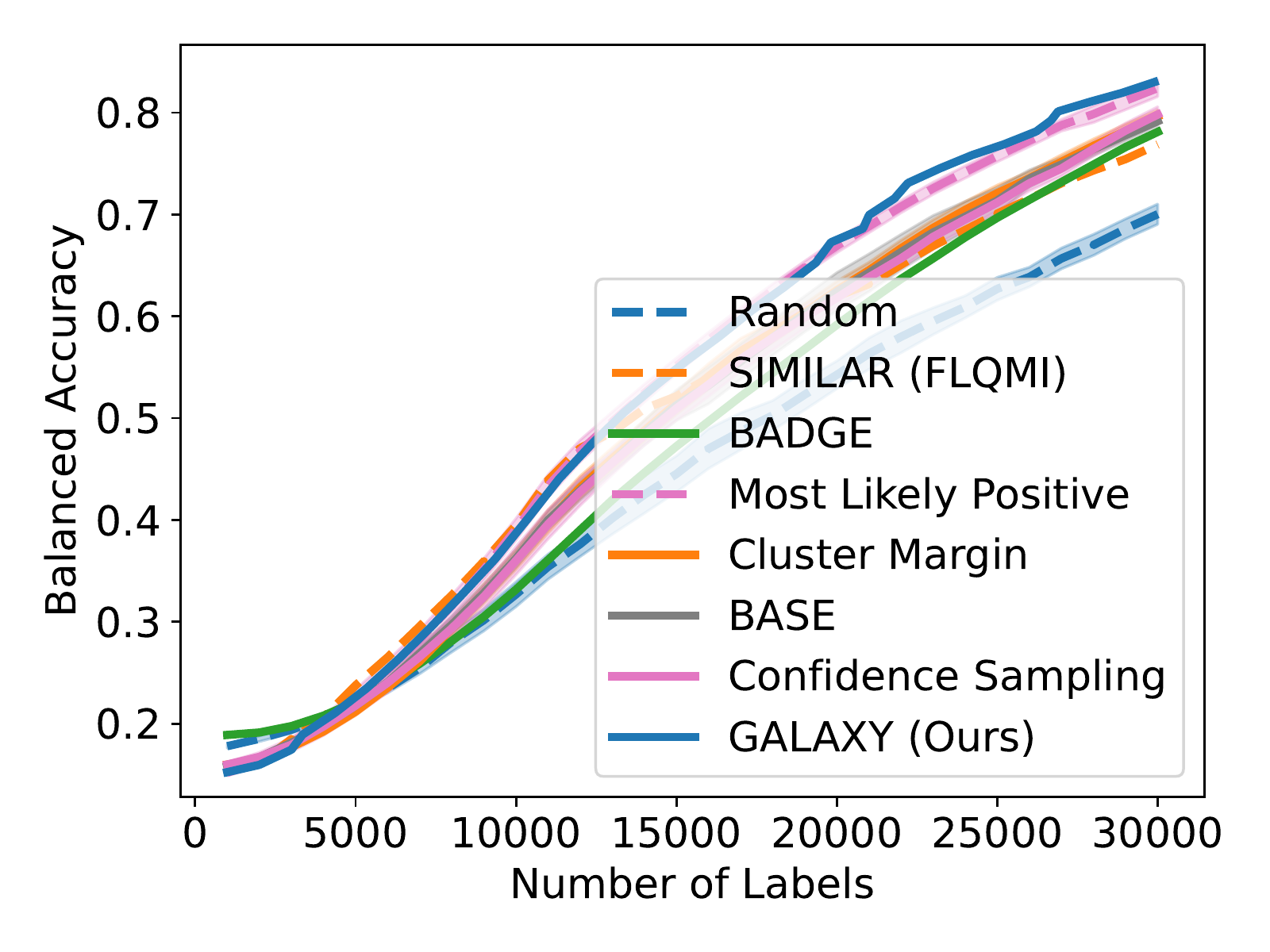}
    \caption{CIFAR-100, 10 classes, batch 1000}
    \label{fig:large_budget}
\end{figure}

\end{document}